\documentclass[Afour,sageh,times]{sagej}
\usepackage{moreverb,url}
\usepackage{times}

\newtheorem{theorem}{\textbf{Theorem}}
\newtheorem{lemma}{\textbf{Lemma}}
\usepackage{color}
\usepackage{times}
\usepackage{epsfig}
\usepackage{epstopdf}
\usepackage{graphicx}
\usepackage{amsmath}
\usepackage{amssymb}
\usepackage{subfigure} 
\usepackage{subfloat}
\usepackage{float}
\usepackage{verbatim}
\usepackage{booktabs}
\usepackage{multirow}
\usepackage{makecell}
\usepackage{colortbl,booktabs}
\usepackage{lscape}
\usepackage{url}
\usepackage{diagbox}
\usepackage{makecell}
\usepackage{xcolor,colortbl}
\usepackage[ruled,vlined,linesnumbered]{algorithm2e}
\usepackage{times}
\usepackage{epsfig}
\usepackage{epstopdf}
\usepackage{graphicx}
\usepackage{amsmath}
\usepackage{amssymb}
\usepackage{soul}
\usepackage{wrapfig}
\usepackage[colorlinks,bookmarksopen,bookmarksnumbered,citecolor=red,urlcolor=red]{hyperref}

\usepackage{lipsum} 
\usepackage{graphicx} 
\usepackage{environ} 
 
\NewEnviron{NORMAL}{%
    \scalebox{0.75}{$\BODY$} 
}

\usepackage{physics}

\DeclareMathAlphabet\mathbfcal{OMS}{cmsy}{b}{n}

\newcommand\BibTeX{{\rmfamily B\kern-.05em \textsc{i\kern-.025em b}\kern-.08em
T\kern-.1667em\lower.7ex\hbox{E}\kern-.125emX}}

\def\volumeyear{2021}

\begin{document}

\runninghead{Siva et al.}

\title{Self-Reflective Terrain-Aware Robot Adaptation for Consistent Off-Road Ground Navigation}
\author{Sriram Siva\affilnum{1}, Maggie Wigness\affilnum{2}, John G. Rogers\affilnum{2}, Long Quang\affilnum{2}, and Hao Zhang\affilnum{1}}
\affiliation{\affilnum{1}Human-Centered Robotics Lab, Colorado School of Mines, Golden, CO 80401, USA.\\
\affilnum{2}Army Research Laboratory (ARL), Adelphi, MD 20783, USA.}
\corrauth{Hao Zhang,
Human-Centered Robotics Lab, Colorado School of Mines, Golden, CO 80401, USA}
\email{hzhang@mines.edu}

\begin{abstract}

\color{black}
Ground robots require the crucial capability of traversing unstructured and unprepared terrains and avoiding obstacles to complete tasks in real-world robotics applications such as disaster response.
When a robot operates in off-road field environments such as forests, 
the robot's actual behaviors often do not match its expected or planned behaviors,
due to changes in the characteristics of terrains and the robot itself.
Therefore, the capability of robot adaptation for consistent behavior generation is essential for maneuverability on unstructured off-road terrains.
In order to address the challenge, 
we propose a novel method of self-reflective terrain-aware adaptation for ground robots
to generate consistent controls to navigate over unstructured off-road terrains,
 which enables robots to more accurately execute the expected behaviors through
robot self-reflection while adapting to varying unstructured terrains.
To evaluate our method's performance, 
we conduct extensive experiments using real ground robots
with various functionality changes over diverse unstructured off-road terrains. 
The comprehensive experimental results have shown that our
self-reflective terrain-aware adaptation method enables ground robots
to generate consistent navigational behaviors and outperforms the compared previous and baseline techniques.  
\end{abstract}


\maketitle

\section{Introduction}

Autonomous ground robots require the crucial capability of navigating over unstructured and unprepared terrains in off-road environments
and avoiding obstacles to complete tasks in real-world applications such as  search and rescue, disaster response, and
reconnaissance  (\cite{kawatsuma2012emergency,park2017disaster,schwarz2017nimbro,kuntze2012seneka}).
When operating in field environments,
ground robots need to navigate over a wide variety of unstructured off-road terrains
with changing types, slope, friction, and other characteristics that cannot be fully modeled beforehand.
In addition, ground robots operating over a long-period of time often experience changes in their own
functionalities, including damages (e.g., failed robot joints and flat
tires), natural wear and tear (e.e., reduced tyre traction), and
varying robot configurations (e.g., varying payload).
As illustrated in Figure \ref{Motivation_IJRR}, such challenges make robot navigation in unstructured off-road environments a challenging problem.
Thus, robot adaptation to changes in terrains and robot functionalities
is essential to the success of offroad ground navigation.

\begin{figure}[h]
\centering
\includegraphics[width=0.49\textwidth]{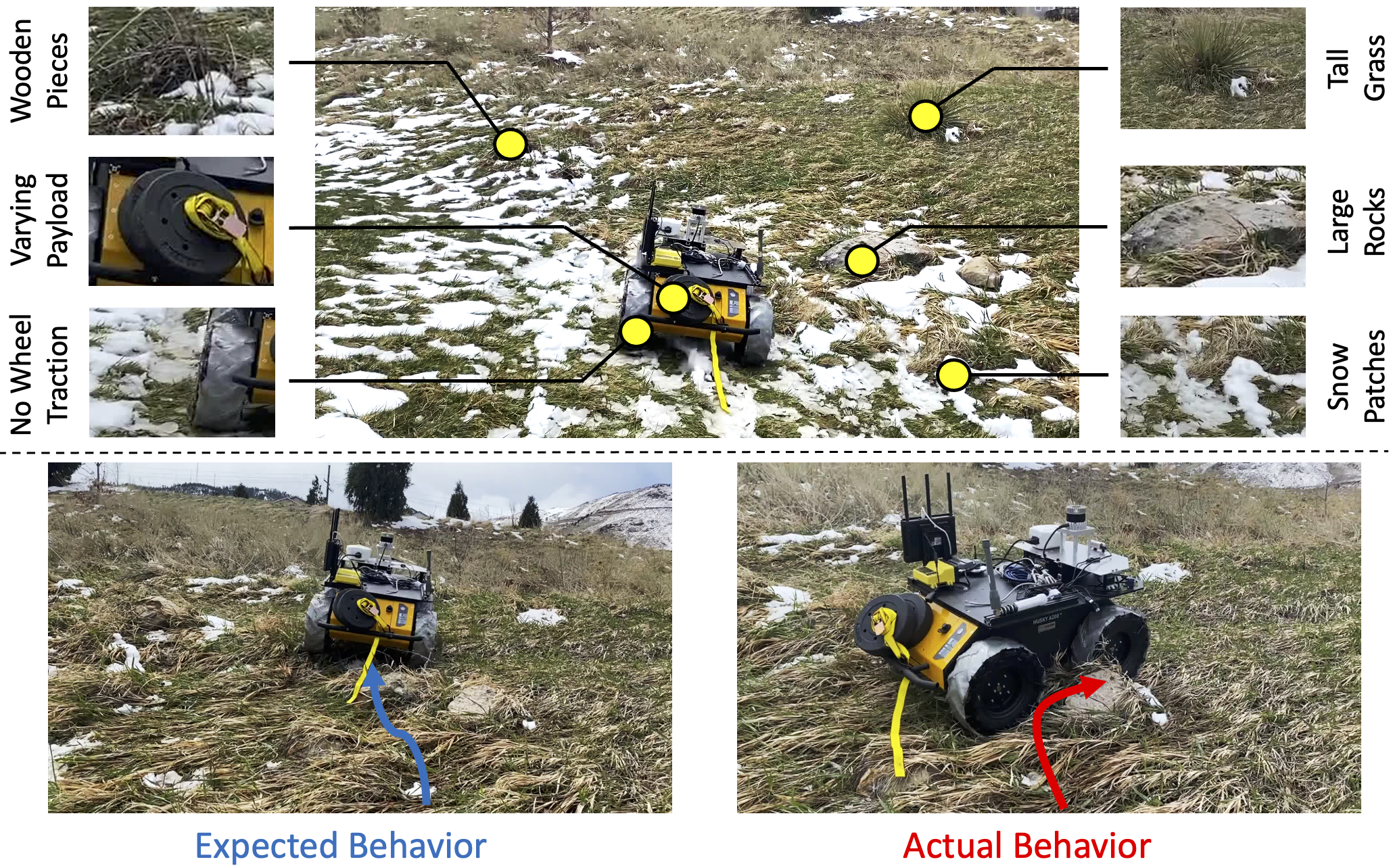}
\caption{
A motivating scenario of self-reflective terrain-aware robot adaptation for generating consistent robot navigational behaviors.
When robots operate in unstructured off-road environments, the actual behaviors often do not match their expected behaviors,
due to changes in the characteristics of terrains and the robots themselves.
Therefore, the capability of robot adaptation for consistent behavior generation is essential for maneuverability over unstructured off-road terrains.
}\label{Motivation_IJRR}
\end{figure}

Given its importance, robot adaptation to unstructured off-road terrains has been an active area of research over the past decades.
Previous methods can be broadly categorized into model-based and learning-based approaches.
Model-based methods often use control theories to model ground robots' dynamics for navigation (\cite{bhattacharyya2018linear, haruno2001mosaic}).
However, they typically ignore terrain characteristics.
Learning-based approaches can learn terrain complexity and estimate the possibility of traversing terrains through terrain recognition  (\cite{bermudez2012performance,peynot2014learned,hudjakov2009aerial}).
Learning-based methods were also developed to adapt navigational behaviors according to the environment
(\cite{posada2010floor,han2017sequence})
and
transfer expert demonstrations on navigational controls to autonomous ground robots
based upon learning from demonstration (LfD) (\cite{silver2010learning,wigness2018robot,siva2019robot}).

Although previous learning-based approaches have shown promising performance,
the expected navigational behaviors generated by previous approaches cannot always be executed accurately by a robot when it navigates over unstructured off-road terrains.
That is --
a robot's actual navigational behaviors may not be consistent with their expected behaviors.
This inconsistency is caused mainly by two reasons.
First, while ground robots navigate in real-world off-road environments,
they experience terrains with a wide variety of characteristics that cannot be modeled beforehand,
for example,
tall grass terrain with hidden rocks as illustrated in Figure \ref{Motivation_IJRR}.
Second, robots may also experience negative effects or the so-called setbacks  (\cite{borges2019strategy,knight2001balancing}), that are defined as the changes in the robot's functionalities that increase the difficulty for a robot to accomplish its expected behaviors.
Examples of the setbacks include malfunctioning robot joints, reduced wheel traction, and heavy payload.
The previous methods were not able to learn to jointly adapt to changes in both terrains and setbacks.

In order to these above shortcomings, we develop a novel method of self-reflective terrain-aware robot adaptation for consistent navigational behavior generation,
which enables robots to accurately execute the expected behaviors through robot self-reflection
while adapting to varying unstructured off-road terrains.
In psychology, self-reflection is considered as the humans' ability to modulate our behaviors by being aware of ourselves (\cite{marcovitch2008self,smith2002self,ardelt2018importance}).
We use the term self-reflection in robotics to refer to the robot's capability of adopting self-awareness
(e.g., a robot knows it has slower speeds than expected) to modulate navigational behavior generation.
In order to enable robot self-reflection,
our method monitors the difference of expected and actual behaviors that is caused by the robot's setbacks,
and accordingly adapts the robot's navigational behaviors to minimize this difference in order to generate consistent navigational behaviors.
Our approach also learns representations of unstructured terrains, which
are jointly used to classify unstructured terrains and generate navigational behaviors.
In addition, our approach is able to fuse historical observations acquired from different onboard sensors to characterize terrains,
and automatically identify the historical observations that provide the most important information for generating consistent behaviors.
The above components are all integrated into a unified mathematical framework of regularized constrained optimization with a theoretical convergence guarantee.

The contribution of this paper focuses on the introduction of the first method of self-reflective terrain-aware adaptation for ground robots to generate consistent navigational controls to traverse
unstructured off-road environments.
The specific novelties of this paper include:
\begin{itemize}
    \item We introduce a novel method for terrain-aware robot adaptation to unstructured off-road terrains,
    which is able to simultaneously recognize terrains and generate corresponding navigational behaviors,
    as well as to identify the most important terrain features for terrain adaptation.
        
    \item We introduce the novel idea and implement one of the first approaches of self-reflection for ground robots to monitor their navigational behaviors and setbacks in order to enhance consistent ground maneuverability.
                
    \item We implement a new optimization algorithm to effectively address the regularized constrained optimization problem that is formulated for self-reflective terrain-aware robot adaptation, which holds a theoretical guarantee to converge to the global optimal solution.
\end{itemize}
Furthermore, as an experimental contribution, 
we perform extensive experiments and provide a comprehensive performance evaluation of 
terrain adaptation methods by designing a set of scenarios for ground robots with various functionality changes over diverse individual and complex unstructured off-road terrain.

The remainder of this paper is structured as follows. A
review of related work is offered in Section \ref{sec:related_work}.
Our proposed methods for self-reflective terrain-aware robot adaptation is discussed in Section \ref{sec:RTA}.
We derive the optimization algorithm in \ref{sec:OPT}.
After discussing the experimental results in Section \ref{sec:EXPT},
we finally conclude the paper in Section \ref{sec:CONC}.

\section{Related Work}\label{sec:related_work}

In this section, we provide a review of the related research in the field of robot adaptation to unstructured terrains.
Existing methods for terrain adaptation can be broadly classified into two categories,
including model-based and learning-based methods.
Previous learning-based approaches can be further divided into methods
for terrain classification and learning-based adaptation.



\subsection{Control Methods for Ground Robot Navigation}


Control theory has been commonly applied to implement ground robot navigation over terrains in robotics.
Many early methods used fuzzy logic to implement robot control for ground navigation over unstructured terrains (\cite{babunski2020application,adib2017mobile,fries2018autonomous}). As these methods didn't consider robot dynamics, methods were proposed to generate navigational controls that also model a robot's dynamics   (\cite{meng1993mobile,kelly1994feedforward,dash2015controlling, verginis2021adaptive, tsunoda2020sheepdog}).
Under control theory, system identification has been commonly considered to model unknown dynamic systems by observing the behavior of the system to various robot  inputs(\cite{johansson1993system,rabiner1978fir}). For the specific case of non-linear robotic systems koopman operator theory (\cite{koopman1931hamiltonian}) provided a way of using infinite dimensional linear observable space to model robot behaviors (\cite{proctor2018generalizing,williams2015data}). 

Robot navigation problems were solved by generating optimal trajectories from robot dynamics model. The trajectory optimization modules are based on dynamic programming, specifically differential dynamic programming and iterative linear quadratic regulator (\cite{yong1999stochastic,tassa2014control,li2004iterative}). These methods are essentially an open loop control and thus susceptible to model errors or disturbances in deployment. To address these problems, model predictive control (MPC) (\cite{camacho2013model}) generates controls given the systems current state by repeatedly solving the optimal control problem over a prediction horizon and only the first optimal control is executed leading to a closed-loop control form.

Although classical control provide abundant models that allow for efficient algorithms for high-dimensional tasks. However, these models can only approximately the dynamics of the mobile robot. To cope with this challenge, learning-based methods obtain robot's dynamics directly from the data.

\subsection{Terrain Classification Methods}
Terrain classification methods make use of robot's sensory data to classify the terrains the robot traverses over. 
Many of the earlier methods were developed to address the navigational challenge of larger vehicles (\cite{bekker1969introduction}) and classification was done in a manual fashion. 
There was an enormous amount of research on terramechanics (\cite{meyerhof1970introduction}), which is the guidance of autonomous vehicles through rough terrains.  
A similar genre of speed selection was used in DARPA grand challenge, to address the problem of unstructured terrain navigation though fixing an optimal speed for each terrain type by collecting data for months (\cite{urmson2004high}). 
Then methods were used to generate digital terrain maps and then fitted a Gaussian model to classify the terrains (\cite{shimoda2005potential,sidek2013exploiting}). 
A few techniques were also designed to use terrain ruggedness data to evaluate navigational behaviors in an online fashion for ground vehicles (\cite{brooks2005vibration,sadhukhan2004terrain}).

In learning based methods, terrain classification was usually performed by learning from terrain data to classify the terrain for navigational behaviors. 
For example, an Support Vector Machine (SVM) based classifier was used on features learned from Hidden Markov Models to identify terrains (\cite{trautmann2011mobility}). 
Color based terrain classification was performed to label obstacles and generate navigational behaviors (\cite{manduchi2005obstacle}). 
More recently, semantic segmentation methods based on deep learning (\cite{long2015fully}) were successfully used to classify off-road terrains (\cite{jiang2020rellis,wigness2019rugd}) and perform the task of terrain navigation. 

Although effective, these methods solely rely in discrete categorization of terrain types. In real world environments, unstructured off-road terrains have a wide variety of characteristics and cannot be categorized into distinctive types. Thus these methods cannot enable in a mobile robot to adapt in real-world  off-road field environments. 

\subsection{Learning-based Adaptation Methods}

This category of methods focus on enabling robots to acutely adapt to the various terrains. High-level models were developed and implemented to solve the general problem of robot terrain adaptation in (\cite{he2019underactuated,nikolaidis2017human,papadakis2013terrain,parker1996alliance,parker2000lifelong}). 
A different genre of methods used case-based reasoning to adapt a robot in dynamic environments (\cite{watson1994case}). 
Many model-based approaches were formulated to learn a single global model using function approximators such as the Gaussian processes (\cite{ko2009gp,deisenroth2011pilco}) and neural networks (\cite{lenz2015deepmpc}). 
Some methods then used probabilistic Gaussian models (\cite{plagemann2008learning}) to learn terrain models and update the models efficiently using sparse approximation. 
The use of ground speed as an optimization variable and formulating methods for trading progress and velocity with changing terrains is observed in earlier works (\cite{fox1997dynamic}). 
Recently, methods were developed that used an array of inertial and ultrasonic sensors to recognize soil properties and perform terrain adaptation (\cite{nabulsi2006multiple}).

With the popularity of learning based methods in the field of robot learning, learning based terrain adaptation methods also started to gain significant attention due to their effectiveness and robustness (\cite{thrun1998lifelong}). 
Terrain adaptation was initially addressed from the perspective of online learning where the model parameters were updated continuously in the execution stage (\cite{kleiner2002towards}). 
Although successful, these methods lacked the ability to adapt on the fly with the changes in unstructured terrains. 
This led to development of open-loop control methods that generated navigational behaviors according to the terrain characteristics (\cite{plagemann2008learning}). 
Reinforcement learning based navigation was employed to generate stable locomotion patterns for terrain adaptation (\cite{erden2008free}). 
Techniques based on inverse reinforcement learning were employed to mimic expert navigational behaviors to achieve near human-level maneuverability (\cite{wigness2018robot}).
More recently, an approach for terrain aware apprenticeship learning has showed promising results for ground robot navigation on unstructured terrains (\cite{siva2019robot}). 

Given the promise of previous learning based terrain adaptation methods, almost all previous approaches focus on learning the expected navigational behaviors, without addressing the need of self-reflection to generate consistent navigational behaviors in these terrains.

\color{black}
\section{Approach}\label{sec:RTA}

In this section, we discuss our novel principled approach to enable ground robots
to adapt their navigational behaviors to unstructured off-road terrains and perform self-reflection to generate consistent navigational behaviors.
An overview of our approach is shown in Figure \ref{Approach_IJRR_2}.

\textbf{Notation:}
We denote scalars using lowercase italic letters (e.g., $m \in \mathbb{R}$),
vectors using boldface lowercase letters (e.g., $\mathbf{m} \in \mathbb{R}^{p}$),
matrices using boldface capital letters, e.g., $\mathbf{M}=\{{m}_{j}^{i}\} \in \mathbb{R}^{p \times q}$ with its $i$-th row and $j$-th column denoted as $\mathbf{m}^{i}$ and $\mathbf{m}_{j}$, respectively.
We use boldface capital Euler script letters to denote tensors (i.e., 3D matrices), e.g., ${\mathbfcal{M}}=\{{{m}}_{j}^{i(k)}\} \in \mathbb{R}^{p \times q \times r}$. Unstacking tensor ${\mathbfcal{M}}$ along its height ($p$), width ($q$) and depth ($r$) provides slices of matrices $\mathbf{M}^{i}$, $\mathbf{M}_{j}$ and $\mathbf{M}^{(k)}$, respectively (\cite{rabanser2017introduction}).
The key variables used in our discussion are defined and briefly explained in Table \ref{tab:variables}.

\begin{table}[hbt]
\centering
\caption{
Definition of the key variables used in our approach.
}
\label{tab:variables}
\tabcolsep=0.2cm
\small
\begin{tabular}{ |c |c|}
\hline
Variables & Definition\\
\hline\hline
\small{$\mathbf{x}_{t} \in \mathbb{R}^{d}$} & \small{Terrain feature vector extracted at time $t$} \\
\small{$\mathbf{X} \in \mathbb{R}^{d \times c}$ } & \small{Terrain feature instance matrix} \\
\small{$\mathbfcal{X} \in \mathbb{R}^{d \times n \times c}$ } & \small{Feature tensor from $n$-different instances} \\
\small{$\mathbf{Z}\in \mathbb{Z}^{l \times n}$ } & \small{Terrain indicator matrix corresponding to $\mathbfcal{X}$} \\
\small{$\mathbf{Y}\in \mathbb{R}^{b \times n}$ } & \small{Expected robot navigational behaviors for $\mathbfcal{X}$} \\
\small{$\mathbf{A} \in \mathbb{R}^{b\times n}$ } & \small{Actual robot navigational behaviors for $\mathbfcal{X}$} \\
\small{$\mathbfcal{E} \in \mathbb{R}^{b \times n \times c}$ } &  \small{Behavior difference tensor} \\
\small{$\mathbfcal{W} \in \mathbb{R}^{l\times d \times c}$ } & \small{Weight tensor to encode terrain features} \\
\small{$\mathbfcal{V}\in \mathbb{R}^{b \times d\times c}$ } & \small{Weight tensor to learn estimated behaviors} \\
\small{$\mathbfcal{U} \in \mathbb{R}^{b\times b \times c}$ } & \small{Weight tensor to learn offset behaviors} \\
\hline
\small{$d$} & \small{Dimensionality of the feature vector}\\
\small{$c$} & \small{Past time steps used to acquire observations}\\
\small{$n$} & \small{Number of instances}\\
\small{$l$} & \small{Number of terrain types}\\
\small{$b$} & \small{Number of robot behavior controls}\\
\hline
\end{tabular}
\end{table}

\color{black}
\subsection{Terrain-Aware Ground Navigation}

We use multiple sensors installed on a ground robot to collect terrain observations,
including RGBD images, LiDAR scans and IMU readings,
at each time step when the robot traverses unstructured terrains.
At each time step $t$,
we extract $m$ types of features (such as visual features and local elevations) from the multi-sensory data.
These features are concatenated into a feature vector denoted as $\mathbf{x}^{(t)} \in \mathbb{R}^{d}$, where $d = \sum_{i=1}^{m}d_{i}$ with $d_i$ denoting the dimensionality of the $i$-th feature type.
Features extracted from a sequence of $c$ past time steps
are stacked into a matrix as a terrain feature instance denoted as $\mathbf{X}=[\mathbf{x}^{(t)};\dots;\mathbf{x}^{(t-c)}] \in \mathbb{R}^{d \times c}$.
To train our approach, we collect a set of $n$ feature instances that are obtained when the robot traverses over various terrains,
and denote this training set as a terrain feature tensor
 $\mathbfcal{X}=[\mathbf{X}_{1},\dots,\mathbf{X}_{n}] \in \mathbb{R}^{d \times n \times c}$.

We denote the terrain types that are associated with $\mathbfcal{X}$ as $\mathbf{Z}=[\mathbf{z}_{1},\dots,{\mathbf{z}_{n}}] \in \mathbb{Z}^{l \times n}$,
where $\mathbf{z}_{i} \in \mathbb{Z}^{l}$ denotes the vector of terrain types
and $l$ represents the number of terrain types.
Each element ${z}_{j}^{i} \in \{ 0,1\}$ indicates whether $\mathbf{X}_{i}$ has the $j$-th terrain type.
During training,
ground truth of the terrain types is provided by human experts.
Given $\mathbfcal{X}$ and $\mathbf{Z}$,
terrain recognition can be formulated as an optimization problem:
\begin{align}\label{eqn1}
& \min_{\mathbfcal{W}} & & \mathcal{L}_{lc}(\mathbfcal{W}\otimes\mathbfcal{X}-\mathbf{Z}) \nonumber \\
& \;\;\, \textrm{s.t.} & & \mathbfcal{W} \otimes \mathbfcal{W}^{\top} = \mathbfcal{I} \\
& & & \mathbfcal{W}\otimes\mathbfcal{W}=\mathbfcal{W}^\top \otimes(\mathbfcal{W}^{\top})^{\top} \nonumber
\end{align}
where $\mathbfcal{W} \in \mathbb{R}^{l\times d \times c}$ is a weight tensor used to encode the importance of each element in $\mathbfcal{X}$ towards estimating terrain types.
Each tensor element $w_{j}^{i(k)} \in \mathbfcal{W} $ denotes the weight of the $j$-th terrain feature from the $k$-th past time step to recognize the $i$-th terrain type.
 $\mathbf{W}^{i}$,
$\mathbf{W}_{j}$ and  $\mathbf{W}^{(k)}$ are the slices of matrices that are obtained by unstacking $\mathbfcal{W}$ along its
height, width and depth, respectively.
$\mathbfcal{I} \in \mathbb{R}^{l \times d \times c}$ is the identity tensor (\cite{qi2017transposes}),
which is used to mathematically define the orthogonality in tensors.
The operator $\otimes$ denotes the vector Kronecker tensor product that performs the matrix-wise multiplication of the two tensors (\cite{neudecker1969note}).
In Eq. (\ref{eqn1}), the tensor product $\otimes$  takes each terrain feature instance $\mathbf{X}_{i} \in \mathbfcal{X}$, and multiplies it with the weight tensor $\mathbfcal{W}$.

We adopt the Log-Cosh loss $\mathcal{L}_{lc}(\cdot) = log(cosh(\cdot))$ (\cite{chen2018log,xu2020towards})
in the objective function in Eq. (\ref{eqn1}) to encode the error of using terrain features $\mathbfcal{X}$ to recognize terrain types $\mathbf{Z}$,
through the learning model that is parameterized by  $\mathbfcal{W}$.
The advantage of using the Log-Cosh loss (e.g., over the $\ell_2$-loss) is twofold.
First, the Log-Cosh loss is robust to outliers,
because a linear error becomes a much smaller value using the logarithmic scale, thus it can reduce the learning bias over outliers.
Because observations acquired by a robot when navigating over unstructured off-road terrain are usually noisy,
the use of the Log-Cosh loss is desirable.
Second, $\mathcal{L}_{lc}(\cdot)$ provides a better optimization of $\mathbfcal{W}$ near the optimality.
This is because, near the optimal value of $\mathbfcal{W}$, the objective value computed by  $\mathcal{L}_{lc}(\cdot)$ in Eq. (\ref{eqn1}) can change significantly when $\mathbfcal{W}$ shows small changes (\cite{chen2018log}).
The two constraints together in Eq. (\ref{eqn1}) are the necessary conditions for orthogonality.
Orthogonality in tensors makes each slice of matrices that stack up a tensor to be full rank,
which means that each slice of matrices is independent and  orthogonal (or perpendicular) to each other and thus sharing the least similarity (\cite{li2018orthogonal}).
In our formulation,
the orthogonality makes $\mathbf{W}^{i}, i=1,\dots, l$ that correspond to all the terrain types to be independent of each other.

\begin{figure*}[t]
\centering
\includegraphics[width=0.995\textwidth]{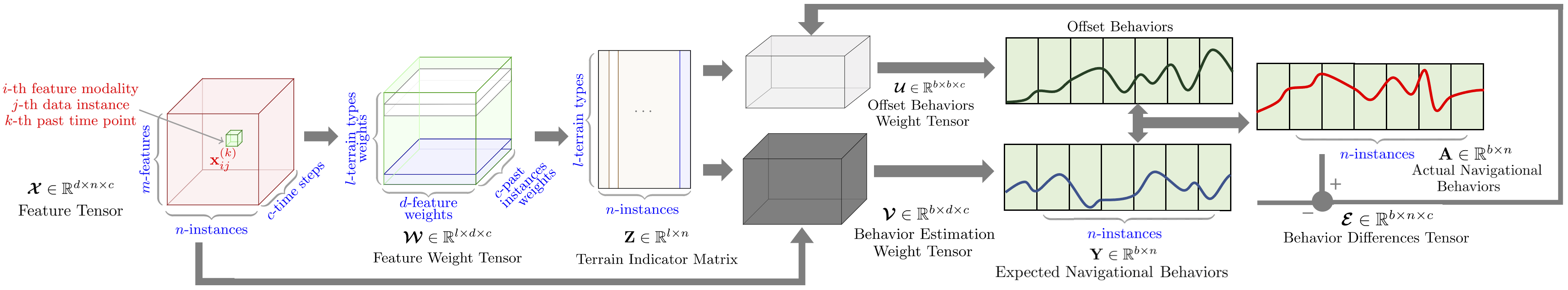}
\caption{
Overview of our proposed method for self-reflective terrain-aware robot adaptation 
to enable ground robots
to adapt their navigational behaviors to unstructured off-road terrains and perform self-reflection to generate consistent navigational behaviors under the unified constrained optimization framework.
}\label{Approach_IJRR_2}
\end{figure*}

The first main novelty of this paper is that we propose a new method for terrain-aware navigation adaptation by joint terrain classification and behavior learning under a unified constrained optimization framework.
This is a achieved by projecting an observation's feature space into a terrain type space and then further projects the terrain type space into a robot navigational behavior space.

Formally, we denote the robot's navigational behaviors as
$\mathbf{Y}=[\mathbf{y}_{1},\dots,\mathbf{y}_{n}] \in \mathbb{R}^{b \times n}$,
where $\mathbf{y}_{i} \in \mathbb{R}^{b}$  represents the navigational behaviors that the robot is expected to execute while traversing over terrain instance $\mathbf{X}_{i}$,
and $b$ denotes the number of distinctive robot controls (e.g., linear and angular velocities).
The expected behaviors can be obtained through demonstrations recorded from human experts when the experts control ground robots to traverse various terrains.
Similar to terrain classification,
navigational behaviors are learned using a history of observations from the pas $c$ time steps,
which allows our formulation to implicitly consider the dynamics of a ground robot.
Then, the problem of terrain-aware navigation adaptation through joint terrain recognition and behavior generation can be formulated as:
\begin{align}\label{eqn2}
& \min_{\mathbfcal{W},\mathbfcal{V}} & & \mathcal{L}_{lc}(\mathbfcal{W}\otimes\mathbfcal{X}-\mathbf{Z}) + \mathcal{L}_{lc}(\mathbfcal{V}\otimes\mathbfcal{W}\otimes\mathbfcal{X}-\mathbf{Y}) \nonumber \\
& \;\;\; \textrm{s.t.} & & \mathbfcal{W}\otimes\mathbfcal{W}^{\top} = \mathbfcal{I} \\
& & & \mathbfcal{W}\otimes\mathbfcal{W}=\mathbfcal{W}^\top \otimes(\mathbfcal{W}^{\top})^{\top} \nonumber
\end{align}
where $\mathbfcal{V} \in \mathbb{R}^{b \times d \times c}$ is a weight tensor that indicates
the importance of the input features $\mathbfcal{X}$ to generate navigational behaviors $\mathbf{Y}$,
and each element $v_{j}^{i(k)}$ represents the weight of using $j$-th terrain feature from $k$-th past time step towards generating the $i$-th navigational behavior.
We further adopt $\mathbf{V}^{i}$, $\mathbf{V}_{j}$ and  $\mathbf{V}^{(k)}$ to denote the slices of matrices obtained by unstacking  $\mathbfcal{V}$ along its height, width and depth, respectively.

The second term in the objective function in Eq. (\ref{eqn2}) is a loss that encodes the difference between the learning model and the demonstrated robot's navigational behaviors.
This loss aims to learn a projection from a history of observations to terrain types,
and then to use the terrain types to generate the robot's behaviors,
thus enabling terrain awareness for our navigational behavior generation method.

Different terrain features capture different characteristics of unstructured terrains,
e.g., color, texture, and slope.
These features typically have different contributions towards terrain recognition and behavior generation.
Similarly, features extracted in specific time steps can be more important than others.
Thus, it is essential to identify the most discriminative features when multiple types of features from a history of time steps are used.
To achieve this capability, we design a new regularization term called the behavior norm, which is mathematically defined as:
\begin{equation}\label{eqn3}
    \Vert \mathbfcal{V}\otimes\mathbfcal{W} \Vert_{B} = \sum_{j=1}^{m}\sum_{k=1}^{c}\Vert  \mathbf{V}_{j}^{(k)}\mathbf{W}_{j}^{(k)} \Vert_{F}
\end{equation}
where $\mathbf{V}_{i}^{(k)}$ and $\mathbf{W}_{i}^{(k)}$ denote the slices of matrices of $\mathbfcal{V}$ and $\mathbfcal{W}$ obtained at
the $k$-th past time step for the $i$-th type of features, respectively.
This norm enforces sparsity between feature modalities obtained over the past $c$ time steps
in order to identify the most discriminative features and time steps during training.

Finally, the problem of terrain-aware behavior generation for terrain adaptation is formulated as:
\begin{align}\label{eqn4}
& \min_{\mathbfcal{W},\mathbfcal{V}} & & \mathcal{L}_{lc}(\mathbfcal{W}\otimes\mathbfcal{X}-\mathbf{Z}) + \mathcal{L}_{lc}(\mathbfcal{V}\otimes\mathbfcal{W}\otimes\mathbfcal{X}-\mathbf{Y}) \nonumber \\
& & & +\lambda_{1}\Vert \mathbfcal{V}\otimes\mathbfcal{W} \Vert_{B}  \nonumber \\
& \;\;\; \textrm{s.t.} & & \mathbfcal{W}\otimes\mathbfcal{W}^{\top} = \mathbfcal{I} \\
& & & \mathbfcal{W}\otimes\mathbfcal{W}=\mathbfcal{W}^\top \otimes(\mathbfcal{W}^{\top})^{\top} \nonumber
\end{align}
where $\lambda_{1}$ is a trade-off hyperparamter to control the amount of regularization.
The problem formulation in Eq. (\ref{eqn4}) allows for terrain-aware navigational behavior generation.
However, ground robots operating in an unstructured environment over a long-period of time often experience changes in their own functionalities, 
such as damages (e.g., failed motors and flat tires), natural wear and tear (e.g., reduced tyre traction), and changed robot configurations (e.g., varying payload).
These negative effects, also called \emph{setbacks} reduce the effectiveness of robot navigation
and cause the robot's actual navigational behaviors not to match the planned or expected behaviors.

\subsection{Self-Reflection for Consistent Navigational Behavior Generation}



The second main novelty of this paper is to propose the idea of self-reflection
for a ground robot to generate consistent navigational behaviors
by adapting to changes in the robot's functionalities.
In psychology, self-reflection is considered as the humans' capability of modulating our behaviors by being aware of ourselves (\cite{marcovitch2008self,smith2002self,ardelt2018importance}).
We use the term self-reflection in robotics to refer to the robot's capability of adopting self-awareness
(e.g., knowing it has slower speeds than expected) to modulate the terrain-aware behavior generation.
Through self-reflection,
we enable a ground robot to generate terrain-aware navigational behaviors that are also conditioned on its self-awareness.
The robot can be self-aware in the way that it continuously monitors the difference between the actual and expected navigational behaviors,
where this difference is caused by setbacks or negative effects.
Then, our method adapts the parameters of terrain-aware behavior generation according to robot's self-awareness,
and computes an offset control to generate consistent navigational behaviors without explicitly modeling the robot setbacks.

Mathematically, we denote the robot's actual navigational behaviors as $\mathbf{A}=[\mathbf{a}_{1},\dots,\mathbf{a}_{n}] \in \mathbb{R}^{b\times n}$,
where $\mathbf{a}_{i} \in \mathbb{R}^{b}$ is the actual navigational behaviors
corresponding to the $i$-th instance $\mathbf{X}_{i}$.
The actual behaviors can be estimated when the ground robot navigates over unstructured terrains using pose estimation methods
such as SLAM or visual odometry (\cite{legoloam2018,shan2016brief}).
Because of setbacks, the robot's actual behaviors usually do not match its expected or planned navigational behaviors.
We compute the difference in navigational behaviors over the past $c$ time steps as
$\mathbf{E} = [({\mathbf{a}}^{(t)}-\mathbf{y}^{(t)}),\dots,(\mathbf{a}^{(t-c)}-\mathbf{y}^{(t-c)})] \in \mathbb{R}^{b \times c}$,
and further denote the behavior differences for all instances $\mathbfcal{X}$
as a tensor  $\mathbfcal{E}=[\mathbf{E}_{1},\dots,\mathbf{E}_{n}] \in \mathbb{R}^{b \times n \times c}$.

Then, we introduce a novel loss function to encode self-reflection for achieving consistent terrain-aware navigational behavior generation as follows:
\begin{equation}\label{eqn5}
\mathcal{L}_{lc}( \mathbfcal{U}\otimes\mathbfcal{W}\otimes\mathbfcal{E} -(\mathbf{A}-\mathbf{Y}))
\end{equation}
where $\mathbfcal{U} = [\mathbf{U}^{(1)},\dots,\mathbf{U}^{(c)}] \in \mathbb{R}^{b\times b \times c}$ is the weight tensor
with $\mathbf{U}^{(k)} \in \mathbb{R}^{b\times c}$ indicating the importance of the $k$-th past behavior difference ${\mathbf{a}}^{(k)}-\mathbf{y}^{(k)}$ to generate offset behaviors $\mathbfcal{U}\otimes\mathbfcal{W}\otimes\mathbfcal{E}$.
Because the loss function in Eq. (\ref{eqn5}) encodes the error between the behavior differences and the generated offset behaviors,
minimizing this loss function is equivalent to minimizing the behavior differences using the generated offset behaviors, or maximizing the consistency of the actual and expected behaviors.
By monitoring the differences over the past $c$ time steps to improve the navigational behavior at the current time step,
this loss function allows our approach to be self-reflective and act as a closed-loop controller
for consistent behavior generation while being terrain aware.

During robot navigation, because of inertia, historical data from the past time steps often contributes differently towards generating offset behaviors.
For example, a heavier payload increases the inertia of the robot; thus a longer time history needs to be considered for generating offset behaviors.
Thus, we introduce an additional regularization term to learn the most informative behavior differences in the past time steps for efficient self-reflection,
which is defined as:
\begin{equation}\label{eqn6}
    \Vert \mathbfcal{U} \Vert_{R} = \sum_{k=1}^{c}\Vert \mathbf{U}^{(k)} \Vert_{F}
\end{equation}
where $\mathbf{U}^{(k)}$ denotes the slice of matrices obtained from $\mathbfcal{U}$ at the $k$-th past time step.

Finally, integrating all components together,
we formulate the problem of self-reflective terrain-aware robot adaptation for generating consistent robot navigational behaviors
as a regularized optimization problem in a unified mathematical framework:
\begin{align}\label{eqn7}
& \min_{\mathbfcal{W},\mathbfcal{V},\mathbfcal{U}} & & \mathcal{L}_{lc}(\mathbfcal{W}\otimes\mathbfcal{X}-\mathbf{Z}) + \mathcal{L}_{lc}(\mathbfcal{V}\otimes\mathbfcal{W}\otimes\mathbfcal{X}-\mathbf{Y}) \nonumber \\
& & & + \mathcal{L}_{lc}( \mathbfcal{U}\otimes\mathbfcal{W}\otimes\mathbfcal{E} -(\mathbf{A}-\mathbf{Y})) \nonumber \\
& & & + \lambda_{1}\Vert \mathbfcal{V}\otimes\mathbfcal{W} \Vert_{B} + \lambda_{2}\Vert \mathbfcal{U} \Vert_{R}  \nonumber \\
& \;\;\; \textrm{s.t.} & & \mathbfcal{W}\otimes\mathbfcal{W}^{\top} = \mathbfcal{I} \\
& & & \mathbfcal{W}\otimes\mathbfcal{W}=\mathbfcal{W}^\top \otimes(\mathbfcal{W}^{\top})^{\top} \nonumber
\end{align}
where $\lambda_{2}$, similar to $\lambda_{1}$, is a hyperparameter to control the trade-off between the loss functions and regularization terms.
An illustration of our complete approach for self-reflective terrain-aware robot adaptation is presented in Figure \ref{Approach_IJRR_2}.

After computing the optimal values of the weight tensors $\mathbfcal{W}$, $\mathbfcal{U}$ and $\mathbfcal{V}$ according to Algorithm 1
in the training phase,
a robot can apply our self-reflective terrain-aware adaptation method to generate consistent navigational behaviors during execution.
At each time step $t$ in the execution phase,
the robot computes multi-model features $\mathbf{X}_{t}$ from observations obtained by its onboard sensors over the past $c$ time steps.
The robot also estimates the corresponding actual behaviors using pose estimation techniques 
(such as SLAM or visual odometry (\cite{legoloam2018,shan2016brief}))
and computes the matrix of the behavior differences $\mathbf{E}_{t}$.
Then our approach can be used by the robot to generate self-reflective terrain-aware navigational behaviors by:
\begin{eqnarray}\label{eq:ControlGen}
    \mathbf{y} = \mathbfcal{V}\otimes\mathbfcal{W}\otimes\mathbf{X}_{t} +  \mathbfcal{U}\otimes\mathbfcal{W}\otimes\mathbf{E}_{t}
\end{eqnarray}
The first term in Eq. (\ref{eq:ControlGen}) generates the navigational behaviors that are aware of terrain types,
which allows a ground robot to adapt its navigational behaviors to unstructured terrains.
The second term in Eq. (\ref{eq:ControlGen}) provides offset controls based on the monitoring of the differences between the expected and actual behaviors
in order to compensate for the setbacks and improve consistency in navigational behaviors.

\begin{algorithm}[t]
    \SetKwInOut{Input}{Input}
    \SetKwInOut{Output}{Output}
    \SetKwInOut{return}{return}

    \Input{$\mathbfcal{X}\in\mathbb{R}^{d\times n \times c}$, $\mathbf{Y}\in\mathbb{R}^{b \times n}$,
    $\mathbf{A}\in\mathbb{R}^{b \times n}$, and
    $\mathbfcal{E}\in\mathbb{R}^{b \times n\times c}$}

    \Output{The weight tensors $\mathbfcal{W}\in \mathbb{R}^{l \times d \times c}$, $\mathbfcal{V}\in \mathbb{R}^{b \times d \times c}$ and $\mathbfcal{U}\in \mathbb{R}^{b \times b \times c}$ }

    Initialize $\mathbfcal{W}\in \mathbb{R}^{l \times d \times c}$, $\mathbfcal{V}\in \mathbb{R}^{b \times d \times c}$ and $\mathbfcal{U}\in \mathbb{R}^{b \times b \times c}$;

    \While{not converge}
    {
    Calculate the block diagonal matrix $\mathbf{Q}^{B}$ with $j$-th column and $k$-th row element  given as $\frac{\mathbf{I}_{1}}{\sum_{i=j}^{m}\sum_{k=1}^{c}\Vert  \mathbf{V}_{j}^{(k)}\mathbf{W}_{j}^{(k)} \Vert_{F}}$;
    
    Calculate the block diagonal matrix $\mathbf{Q}^{l}$ with $i$-th diagonal block given as $\frac{\mathbf{I}^{l}}{\Vert \mathbf{W}^{i}(\mathbf{W}^{i})^{\top} - \mathbf{I}^{l}\Vert_{F}}$;
     
    Calculate the block diagonal matrix $\mathbf{Q}_{d}$ with $j$-th diagonal block given as $\frac{\mathbf{I}_{d}}{\Vert\mathbf{W}_{j}(\mathbf{W}_{j})^{\top} - \mathbf{I}_{d}\Vert_{F}}$;
     
    Calculate the block diagonal matrix $\mathbf{Q}^{(c)}$ with $k$-th diagonal block given as $\frac{\mathbf{I}^{(c)}}{\Vert\mathbf{W}^{(k)}(\mathbf{W}^{(k)})^{\top} - \mathbf{I}^{c}\Vert_{F}}$;
     
    Compute each slice of matrices $\mathbf{W}^{i}$ from Eq. (\ref{eqn11});
    
    Calculate the block diagonal matrix $\mathbf{O}^{B}$ with $j$-th column and $k$-th row element  given as $\frac{\mathbf{I}_{1}}{\sum_{i=j}^{m}\sum_{k=1}^{c}\Vert  \mathbf{V}_{j}^{(k)}\mathbf{W}_{j}^{(k)} \Vert_{F}}$;  
    
    Compute the matrix $\mathbf{V}^{i}$ from Eq. (\ref{eqn12});  
      
    Calculate the block diagonal matrix $\mathbf{P}^{(R)}$ with $k$-th diagonal block given as $\frac{\mathbf{I}_{2}}{\sum_{k=t}^{t-c}\Vert  \mathbf{U}^{(k)} \Vert_{F}}$;

    Compute the matrix $\mathbf{U}^{i}$ from Eq. (\ref{eqn13});
    }
    \textbf{return: } $\mathbfcal{W}\in \mathbb{R}^{l \times d \times c}$, $\mathbfcal{V}\in \mathbb{R}^{b \times d \times c}$ and $\mathbfcal{U}\in \mathbb{R}^{b \times b \times c}$

 \caption{The implemented algorithm to solve the formulated constrained regularized optimization problem in Eq. (\ref{eqn7}).} 
 \label{alg1}
\end{algorithm}

\section{Optimization Algorithm}\label{sec:OPT}

In this section, we implement a new iterative algorithm, 
as shown in in Algorithm \ref{alg1}, to compute the optimal solution to the formulated constrained regularized optimization problem in Eq. (\ref{eqn7}).
The optimization problem is challenging to solve because of the constrains added to the objective function and the two non-smooth structured regularization terms.

To solve the constrained optimization problem in Eq. (\ref{eqn7}), we first convert the constrains into a matrix form as:
\vspace{-10pt}
\begin{align}\label{eqn8}
\text{\footnotesize{$\min_{\mathbfcal{W},\mathbfcal{V},\mathbfcal{U}}$}} &  \scriptstyle  \mathcal{L}_{lc}(\mathbfcal{W}\otimes\mathbfcal{X}-\mathbf{Z}) + \mathcal{L}_{lc}(\mathbfcal{V}\otimes\mathbfcal{W}\otimes\mathbfcal{X}-\mathbf{Y}) \nonumber  \\ &\scriptstyle +\mathcal{L}_{lc}( \mathbfcal{U}\otimes\mathbfcal{W}\otimes\mathbfcal{E} -(\mathbf{A}-\mathbf{Y}))
     +\lambda_{1}\Vert \mathbfcal{V}\otimes\mathbfcal{W} \Vert_{B} + \lambda_{2}\Vert \mathbfcal{U} \Vert_{R} \nonumber \\
     \textrm{\footnotesize{s.t}} \quad &\scriptstyle \mathbf{W}^{i}(\mathbf{W}^{i})^{\top} = \mathbf{I}^{l}; \forall i = 1,\dots,l   \\
& \scriptstyle \mathbf{W}_{j}(\mathbf{W}_{j})^{\top} = \mathbf{I}_{d}; \forall j = 1,\dots,d \nonumber  \\
& \scriptstyle \mathbf{W}^{(k)}(\mathbf{W}^{(k)})^{\top} = \mathbf{I}^{(c)}; \forall k = 1,\dots,c \nonumber
\end{align}
where $\mathbf{I}^{l}$, $ \mathbf{I}_{d}$ and $\mathbf{I}^{(c)}$ are identity matrices of size $l,d,c$.
In Eq. (\ref{eqn8}), the constrains are written using a matrix form,
This form still enforces orthogonality, the same as the tensor form in Eq. (\ref{eqn7}),
as the matrix-form constraints make each column in each slice of matrices $\mathbf{W}^{i}$, $\mathbf{W}_{j}$ and $\mathbf{W}^{(k)}$ to be full rank.
Thus, each slice of matrix in tensor $\mathbfcal{W}$ is also full rank and $\mathbfcal{W}$ is orthogonal. 

Then we can rewrite the objective function in Eq. (\ref{eqn8}) using Lagrangian multipliers $\lambda^{l}, \lambda_{d}$ and $\lambda^{(c)}$ as:
\begin{align}\label{eqn9}
\text{\footnotesize{$\min_{\mathbfcal{W},\mathbfcal{V},\mathbfcal{U}}$}} &  \scriptstyle \mathcal{L}_{lc}(\mathbfcal{W}\otimes\mathbfcal{X}-\mathbf{Z}) \scriptstyle +  \mathcal{L}_{lc}(\mathbfcal{V}\otimes\mathbfcal{W}\otimes\mathbfcal{X}-\mathbf{Y})
\scriptstyle +  \scriptstyle \mathcal{L}_{lc}( \mathbfcal{U}\otimes\mathbfcal{W}\otimes\mathbfcal{E} -(\mathbf{A}-\mathbf{Y})) \nonumber \\ \scriptstyle + &  \scriptstyle \lambda_{1}\Vert \mathbfcal{V}\otimes\mathbfcal{W} \Vert_{B} + \lambda_{2}\Vert \mathbfcal{U} \Vert_{R} +  \Sigma_{i=1}^{l}\lambda^{l} \Vert \mathbf{W}^{i}(\mathbf{W}^{i})^{\top} - \mathbf{I}^{l}\Vert_{F}  \\
\scriptstyle + & \scriptstyle \Sigma_{j=1}^{d}\lambda_{d} \Vert\mathbf{W}_{j}(\mathbf{W}_{j})^{\top} - \mathbf{I}_{d}\Vert_{F}+ \Sigma_{k=1}^{c} \lambda^{(c)} \Vert \mathbf{W}^{(k)}(\mathbf{W}^{(k)})^{\top} - \mathbf{I}^{(c)}\Vert_{F}   \nonumber
\end{align}
Each of the Lagrangian multipliers is designed to have a high value (i.e., $\lambda^{l}, \lambda_{d}, \lambda^{(c)} >\!\!> 0$), such that even a small variation in orthogonality in each slices of matrix from weight tensor $\mathbfcal{W}$ can result in a high cost. To compute the optimal weight tensor $\mathbfcal{W}$, we minimize Eq. (\ref{eqn9}) with respect to $\mathbf{W}^{i}$, $i=1,\dots,l$, resulting in:
\begin{align}\label{eqn11}
&\scriptstyle \tanh{(\mathbfcal{W}\otimes\mathbfcal{X}-\mathbf{Z})}((\mathbf{X}^{i})^\top \mathbf{W}^{i})  \tanh{(\mathbfcal{V}\otimes\mathbfcal{W}\otimes\mathbfcal{X}-\mathbf{Y})}((\mathbf{X}^{i}\mathbf{V}^{i})^\top\mathbf{W}^{i}) \nonumber \\
& \scriptstyle + \tanh{( \mathbfcal{U}\otimes\mathbfcal{W}\otimes\mathbfcal{E} -(\mathbf{A}-\mathbf{Y}))}((\mathbf{E}^{i}\mathbf{V}^{i})^\top \mathbf{W}^{i}) \\
& \scriptstyle +(\lambda_{1}\mathbf{Q}^{B}+\Sigma_{i=1}^{l}\lambda^{l}\mathbf{Q}^{l}+\Sigma_{j=1}^{d}\lambda_{d}\mathbf{Q}_{d}+\Sigma_{k=1}^{c}\lambda^{(c)}\mathbf{Q}^{(c)})\mathbf{W}^{i}=0 \nonumber
\end{align}
\normalsize{where $\mathbf{Q}^{B},\mathbf{Q}^{l},\mathbf{Q}_{d}$ and $\mathbf{Q}^{(c)}$ denote block diagonal matrices that are dependent on the weight tensor $\mathbfcal{W}$. Mathematically, we express each element in the $j$-th column and $k$-th row in $\mathbf{Q}^{B}$ as $\frac{1}{\sum_{i=j}^{m}\sum_{k=1}^{c}\Vert\mathbf{V}_{j}^{(k)}\mathbf{W}_{j}^{(k)}\Vert_{F}}$. 
Each of the $i$-th block in $\mathbf{Q}^{l}$ is given by $\frac{\mathbf{I}^{i}}{\Vert \mathbf{W}^{i}(\mathbf{W}^{i})^{\top}-\mathbf{I}^{l}\Vert_{F}}$. The $j$-th diagonal block in $\mathbf{Q}_{d}$ is given as $\frac{\mathbf{I}_{d}}{\Vert\mathbf{W}_{j}(\mathbf{W}_{j})^{\top} - \mathbf{I}_{d}\Vert_{F}}$ and each block diagonal element in $\mathbf{Q}^{(c)}$ is given by $\frac{\mathbf{I}^{(c)}}{\Vert\mathbf{W}^{(k)}(\mathbf{W}^{(k)})^{\top} - \mathbf{I}^{(c)}\Vert_{F}}$. Because each slice of matrix $\mathbf{W}^{i}$ and the block diagonal matrices $\mathbf{Q}^{B},\mathbf{Q}^{l},\mathbf{Q}_{d}$, $\mathbf{Q}^{(c)}$ are interdependent,
we need an iterative algorithm to compute them. }

We use the optimal slices of weight matrices $\mathbf{W}^{i}$ to solve weight tensors $\mathbfcal{V}$ and $\mathbfcal{U}$. Accordingly, we differentiate Eq. (\ref{eqn9}) with respect to $\mathbf{V}^{i}$ and $\mathbf{U}^{i}$, $i=1,\dots,b$, results in:
\begin{align}\label{eqn12}
&\scriptstyle \tanh{(\mathbfcal{V}\otimes\mathbfcal{W}\otimes\mathbfcal{X}-\mathbf{Y})}((\mathbf{W}^{i}\mathbf{X}^{i})^\top\mathbf{V}^{i}-\mathbf{WX}^{i}\mathbf{y}^{i}) + \lambda_{1}\mathbf{O}^{B} = 0
\end{align}
and
\begin{align}\label{eqn13}
&\scriptstyle \tanh{( \mathbfcal{U}\otimes\mathbfcal{W}\otimes\mathbfcal{E} -(\mathbf{A}-\mathbf{Y}))}((\mathbf{W}^{i}\mathbf{E}^{i})^\top \mathbf{U}^{i}-\mathbf{E}^{i})
+ \lambda_{2}\mathbf{P}^{(R)} = 0
\end{align}
where $\mathbf{O}^{(B)}$ and $\mathbf{P}^{(R)}$ represent block diagonal matrices that are dependent on $\mathbfcal{V}$ and $\mathbfcal{U}$ respectively. 
Mathematically, we calculate each element in the $j$-th column and $k$-th row in $\mathbf{O}^{B}$ as $\frac{1}{\sum_{i=j}^{m}\sum_{k=1}^{c}\Vert\mathbf{V}_{j}^{(k)}\mathbf{W}_{j}^{(k)}\Vert_{F}}$.
The $k$-th block diagnoal element in $\mathbf{P}^{(R)}$ is calculated by  $\frac{\mathbf{I}_{(k)}}{\sum_{k=1}^{c}\Vert  \mathbf{U}^{(k)} \Vert_{F}}$.
The optimal values of $\mathbf{V}^{i}$ and $\mathbf{U}^{i}$ are employed to update $\mathbf{W}^{i}$ in the next iteration.
Similarly, $\mathbf{U}^{i}$ and $\mathbf{P}$ are interdependent, and thus an iterative algorithm is needed to update their values.

In the following, we prove that Algorithm \ref{alg1} decreases the value of the objective function in Eq. (\ref{eqn7}) with each iteration and converges to the global optimal solution. But first, we present a lemma from \cite{nie2010efficient}:

\begin{lemma}\label{lemma1}
\normalsize{Given any two matrices $\mathbf{A}$ and $\mathbf{B}$, the following inequality relation holds:}
\begin{equation}
\Vert\mathbf{B}\Vert_F - \frac{\Vert\mathbf{B}\Vert_F^2}{2\Vert\mathbf{A}\Vert_F}
\leq
\Vert\mathbf{A}\Vert_F - \frac{\Vert\mathbf{A}\Vert_F^2}{2\Vert\mathbf{A}\Vert_F}
\end{equation}
\end{lemma}
\begin{theorem}\label{thm1}
\normalsize{Algorithm \ref{alg1} converges to the global optimal solution of the constrained regularized optimization problem in Eq. (\ref{eqn7})}.
\end{theorem}
\begin{proof}
\normalsize{According to Step 7 in Algorithm \ref{alg1}, the value of $\mathbf{W}^{i}(s+1)$ is computed from $\mathbfcal{W}(s)$ in the $s$-th iteration by:}
\begin{align}\label{proof1_1}
\scriptstyle \mathbf{W}^{i}(s+1)& \scriptstyle =\mathcal{L}_{lc}(\mathbfcal{W}\otimes\mathbfcal{X}-\mathbf{Z})
+\lambda_{1}Tr  (\mathbf{V}^{i}\mathbf{W}^{i\top}) \mathbf{Q}^{B}(s+1) (\mathbf{V}^{i}\mathbf{W}^{i}) \nonumber \\
\scriptstyle +\lambda^{l}Tr & \scriptstyle (\mathbf{W}^{i}\mathbf{W}^{i\top}-\mathbf{I}^{l^\top}) \mathbf{Q}^{l}(s+1)(\mathbf{W}^{i}\mathbf{W}^{i\top}-\mathbf{I}^{l})  \\
\scriptstyle +\lambda_{d} Tr & \scriptstyle (\mathbf{W}_{j}\mathbf{W}_{j}^{\top}-\mathbf{I}_{d})^\top \mathbf{Q}_{d}(s+1)(\mathbf{W}_{j}\mathbf{W}_{j}^{\top}-\mathbf{I}^{d}) \nonumber \\
\scriptstyle +\lambda^{(c)}Tr & \scriptstyle (\mathbf{W}^{(k)}\mathbf{W}^{(k)\top} - \mathbf{I}^{(c)\top}) \mathbf{Q}^{(c)}(s+1) (\mathbf{W}^{(k)}\mathbf{W}^{(k)\top} - \mathbf{I}^{(c)}) \nonumber
\end{align}
where the operator $Tr$ denotes the trace of a matrix. From Step 9 in Algorithm \ref{alg1}, we obtain:
\begin{align}\label{proof2}
\scriptstyle  \mathbf{V}^{i}(s+1)\scriptstyle  =\mathcal{L}_{lc}(\mathbfcal{V}\otimes\mathbfcal{W}\otimes\mathbfcal{X}-\mathbf{Y}) \scriptstyle  +\lambda_{1}Tr(\mathbf{V}^{i}\mathbf{W}^{i})^\top \mathbf{O}^{B}(s+1) (\mathbf{V}^{i}\mathbf{W}^{i})
\end{align}
From Step 11 in Algorithm \ref{alg1} we also obtain:
\begin{align}\label{proof3}
\scriptstyle \mathbf{U}^{i}(s+1)=\mathcal{L}_{lc}(\mathbfcal{U}\otimes\mathbfcal{W}\otimes\mathbfcal{E}-(\mathbf{A}-\mathbf{Y}))
+\lambda_{2}Tr \mathbf{U}^{i}\mathbf{P}(s+1)\mathbf{U}^{i}
\end{align}
Then, we derive that:
\begin{align}\label{proof4}
&\scriptstyle \mathbfcal{F}_{1}(s+1)+\mathbfcal{F}_{2}(s+1)+\mathbfcal{F}_{3}(s+1) +\lambda_{1}Tr(\mathbf{V}^{i}\mathbf{W}^{i})^\top \mathbf{Q}^{B}(s+1) (\mathbf{V}^{i}\mathbf{W}^{i})  \nonumber \\
&\scriptstyle +\lambda^{l}Tr (\mathbf{W}^{i}\mathbf{W}^{i\top}-\mathbf{I}^{l})^\top \mathbf{Q}^{l}(s+1)(\mathbf{W}^{i}\mathbf{W}^{i\top}-\mathbf{I}^{l}) \nonumber \\
&\scriptstyle +\lambda_{d} Tr (\mathbf{W}_{j}\mathbf{W}_{j}^{\top}-\mathbf{I}_{d})^\top \mathbf{Q}_{d}(s+1)(\mathbf{W}_{j}\mathbf{W}_{j}^{\top}-\mathbf{I}^{d}) \nonumber \\
&\scriptstyle +\lambda^{(c)}Tr(\mathbf{W}^{(k)}\mathbf{W}^{(k)\top} - \mathbf{I}^{(c)\top} \mathbf{Q}^{(c)}(s+1) (\mathbf{W}^{(k)}\mathbf{W}^{(k)\top} - \mathbf{I}^{(c)}) \nonumber \\
&\scriptstyle +\lambda_{1}Tr(\mathbf{V}^{i}\mathbf{W}^{i})^\top \mathbf{O}^{B}(s+1) (\mathbf{V}^{i}\mathbf{W}^{i})
+\lambda_{2}Tr\mathbf{U}^{i}\mathbf{P}(s+1)\mathbf{U}^{i} \nonumber \\
& \scriptstyle \leq \mathbfcal{F}_{1}(s)+\mathbfcal{F}_{2}(s)+\mathbfcal{F}_{3}(s) +\lambda_{1}Tr(\mathbf{V}^{i}\mathbf{W}^{i})^\top \mathbf{Q}^{B}(s) (\mathbf{V}^{i}\mathbf{W}^{i})  \\
&\scriptstyle +\lambda^{l}Tr (\mathbf{W}^{i}\mathbf{W}^{i\top}-\mathbf{I}^{l})^\top \mathbf{Q}^{l}(s)(\mathbf{W}^{i}\mathbf{W}^{i\top}-\mathbf{I}^{l}) \nonumber \\
&\scriptstyle +\lambda_{d} Tr (\mathbf{W}_{j}\mathbf{W}_{j}^{\top}-\mathbf{I}_{d})^\top \mathbf{Q}_{d}(s)(\mathbf{W}_{j}\mathbf{W}_{j}^{\top}-\mathbf{I}^{d}) \nonumber \\
&\scriptstyle +\lambda^{(c)}Tr(\mathbf{W}^{(k)}\mathbf{W}^{(k)\top} - \mathbf{I}^{(c)})^\top \mathbf{Q}^{(c)}(s) (\mathbf{W}^{(k)}\mathbf{W}^{(k)\top} - \mathbf{I}^{(c)}) \nonumber \\
&\scriptstyle +\lambda_{1}Tr(\mathbf{V}^{i}\mathbf{W}^{i})^\top \mathbf{O}^{B}(s) (\mathbf{V}^{i}\mathbf{W}^{i})
+\lambda_{2}Tr\mathbf{U}^{i}\mathbf{P}(s)\mathbf{U}^{i} \nonumber
\end{align}
where:
\begin{align}
&\mathbfcal{F}_{1}(s)=\mathcal{L}_{lc}(\mathbfcal{W}(s)\otimes\mathbfcal{X}-\mathbf{Z}) \nonumber \\
&\mathbfcal{F}_{2}(s)= \mathcal{L}_{lc}(\mathbfcal{V}(s)\otimes\mathbfcal{W}(s)\otimes\mathbfcal{X}-\mathbf{Y}) \nonumber \\
&\mathbfcal{F}_{3}(s) = \mathcal{L}_{lc}( \mathbfcal{U}(s)\otimes\mathbfcal{W}(s)\otimes\mathbfcal{E} -(\mathbf{A}-\mathbf{Y})) \nonumber
\end{align}
After substituting $\mathbf{Q}^{B},\mathbf{Q}^{l},\mathbf{Q}_{d}$, $\mathbf{Q}^{(c)}$, $\mathbf{O}^{B}$ and $\mathbf{P}^{R}$ in Eq. (\ref{proof4}), we obtain:
\begin{align}\label{proof5}
&\scriptstyle \mathbfcal{F}_{1}(s+1)+\mathbfcal{F}_{2}(s+1)+\mathbfcal{F}_{3}(s+1) +\lambda_{2}\sum_{k=1}^{c}\frac{\Vert\mathbf{U}^{(k)}(s+1)\Vert_{F}^{2}}{2\Vert\mathbf{U}^{(k)}(s)\Vert_{F}} \nonumber \\
&\scriptstyle +\lambda_{1}\sum_{i=j}^{m}\sum_{k=1}^{c}\frac{\Vert\mathbf{V}_{j}^{(k)}(s)\mathbf{W}_{j}^{(k)}(s+1)\Vert_{F}^{2}}{2\Vert\mathbf{V}_{j}^{(k)}(s)\mathbf{W}_{j}^{(k)}(s)\Vert_{F}}  \nonumber \\
&\scriptstyle +\lambda^{l}\sum_{i=1}^{l}\frac{\Vert \mathbf{W}^{i}(s+1)(\mathbf{W}(s+1)^{i})^{\top}-\mathbf{I}^{l}\Vert_{F}^{2}}{\Vert 2 \mathbf{W}^{i}(s)(\mathbf{W}^{i}(s))^{\top}-\mathbf{I}^{l}\Vert_{F}} \nonumber \\
&\scriptstyle +\lambda_{d}\sum_{j=1}^{d}\frac{\Vert \mathbf{W}_{j}(s+1)(\mathbf{W}_{j}(s+1))^{\top}-\mathbf{I}_{d}\Vert_{F}^{2}}{\Vert 2 \mathbf{W}_{j}(s)(\mathbf{W}_{j}(s))^{\top}-\mathbf{I}_{d}\Vert_{F}} \nonumber \\
&\scriptstyle +\lambda^{(c)}\sum_{k=1}^{c}\frac{\Vert \mathbf{W}^{(k)}(s+1)(\mathbf{W}^{(k)}(s+1))^{\top}-\mathbf{I}^{(c)}\Vert_{F}^{2}}{\Vert  2\mathbf{W}^{(k)}(s)(\mathbf{W}^{(k)}(s))^{\top}-\mathbf{I}^{(c)}\Vert_{F}}  \nonumber \\
& \scriptstyle \leq \mathbfcal{F}_{1}(s)+\mathbfcal{F}_{2}(s)+\mathbfcal{F}_{3}(s)+\lambda_{2}\sum_{k=1}^{c}\frac{\Vert\mathbf{U}^{(k)}(s)\Vert_{F}^{2}}{2\Vert\mathbf{U}^{(k)}(s)\Vert_{F}} \\
&\scriptstyle +\lambda_{1}\sum_{i=j}^{m}\sum_{k=1}^{c}\frac{\Vert\mathbf{V}_{j}^{(k)}(s)\mathbf{W}_{j}^{(k)}(s)\Vert_{F}^{2}}{2\Vert\mathbf{V}_{j}^{(k)}(s)\mathbf{W}_{j}^{(k)}(s)\Vert_{F}}  \nonumber \\
&\scriptstyle +\lambda^{l}\sum_{i=1}^{l}\frac{\Vert \mathbf{W}^{i}(s)(\mathbf{W}(s)^{i})^{\top}-\mathbf{I}^{l}\Vert_{F}^{2}}{\Vert 2 \mathbf{W}^{i}(s)(\mathbf{W}^{i}(s))^{\top}-\mathbf{I}^{l}\Vert_{F}} \nonumber \\
&\scriptstyle +\lambda_{d}\sum_{j=1}^{d}\frac{\Vert \mathbf{W}_{j}(s)(\mathbf{W}_{j}(s))^{\top}-\mathbf{I}_{d}\Vert_{F}^{2}}{\Vert 2 \mathbf{W}_{j}(s)(\mathbf{W}_{j}(s))^{\top}-\mathbf{I}_{d}\Vert_{F}} \nonumber \\
&\scriptstyle +\lambda^{(c)}\sum_{k=1}^{c}\frac{\Vert \mathbf{W}^{(k)}(s)(\mathbf{W}^{(k)}(s))^{\top}-\mathbf{I}^{(c)}\Vert_{F}^{2}}{\Vert  2\mathbf{W}^{(k)}(s)(\mathbf{W}^{(k)}(s))^{\top}-\mathbf{I}^{(c)}\Vert_{F}}  \nonumber
\end{align}
Using Lemma \ref{lemma1}, we obtain the following  inequalities in Eqs. (\ref{proof6}-\ref{proof10}).
From $\mathbf{Q}^{B}$ and $\mathbf{O}^{B}$, we obtain:
\begin{align}\label{proof6}
&\scriptstyle \sum_{i=1}^{m}\sum_{k=1}^{c}\Big( \Vert\mathbf{V}_{j}^{(k)}(s)\mathbf{W}_{j}^{(k)}(s+1)\Vert_{F} - \frac{\Vert\mathbf{V}_{j}^{(k)}(s)\mathbf{W}_{j}^{(k)}(s+1)\Vert_{F}^{2}}{2\Vert\mathbf{V}_{j}^{(k)}(s)\mathbf{W}_{j}^{(k)}(s)\Vert_{F}} \Big) \nonumber \\
&\scriptstyle \leq
\sum_{i=1}^{m}\sum_{k=1}^{c}\Big( \Vert\mathbf{V}_{j}^{(k)}(s)\mathbf{W}_{j}^{(k)}(s)\Vert_{F} - \frac{\Vert\mathbf{V}_{j}^{(k)}(s)\mathbf{W}_{j}^{(k)}(s)\Vert_{F}^{2}}{2\Vert\mathbf{V}_{j}^{(k)}(s)\mathbf{W}_{j}^{(k)}(s)\Vert_{F}} \Big)
\end{align}
From $\mathbf{Q}^{l}$, we obtain:
\begin{align}\label{proof7}
&\scriptstyle \sum_{i=1}^{m}\Big( \Vert \mathbf{W}^{i}(s+1)\mathbf{W}^{i\top}(s+1)-\mathbf{I}^{l}\Vert_{F} - \frac{\Vert \mathbf{W}^{i}(s+1)\mathbf{W}^{i\top}(s+1)-\mathbf{I}^{l}\Vert_{F}^{2}}{2\Vert\mathbf{W}^{i}(s)\mathbf{W}^{i\top}(s)-\mathbf{I}^{l}\Vert_{F}}  \Big) \nonumber \\
&\scriptstyle \leq
\sum_{i=1}^{m}\Big(\Vert \mathbf{W}^{i}(s)\mathbf{W}^{i\top}(s)-\mathbf{I}^{l}\Vert_{F} - \frac{\Vert \mathbf{W}^{i}(s)\mathbf{W}^{i\top}(s)-\mathbf{I}^{l}\Vert_{F}^{2}}{2\Vert\mathbf{W}^{i}(s)(\mathbf{W}^{i\top}(s))-\mathbf{I}^{l}\Vert_{F}} \Big)
\end{align}
From $\mathbf{Q}_{d}$, we obtain:
\begin{align}\label{proof8}
&\scriptstyle \sum_{j=1}^{d}\Big( \Vert \mathbf{W}_{j}(s+1)(\mathbf{W}_{j}(s+1))^{\top}-\mathbf{I}_{d}\Vert_{F} - \frac{\Vert \mathbf{W}_{j}(s+1)(\mathbf{W}_{j}(s+1))^{\top}-\mathbf{I}_{d}\Vert_{F}^{2}}{2\Vert\mathbf{W}_{j}(s)(\mathbf{W}_{j}(s))^{\top}-\mathbf{I}_{d}\Vert_{F}}  \Big) \nonumber \\
&\scriptstyle \leq
\sum_{j=1}^{d}\Big(\Vert \mathbf{W}_{j}(s)(\mathbf{W}_{j}(s))^{\top}-\mathbf{I}_{d}\Vert_{F} - \frac{\Vert \mathbf{W}_{j}(s)(\mathbf{W}_{j}(s))^{\top}-\mathbf{I}_{d}\Vert_{F}^{2}}{2\Vert\mathbf{W}_{j}(s)(\mathbf{W}_{j}(s))^{\top}-\mathbf{I}_{d}\Vert_{F}} \Big)
\end{align}
From $\mathbf{Q}^{(c)}$, we obtain:
\begin{align}\label{proof9}
&\scriptstyle  \sum_{k=1}^{c}\Big( \Vert \mathbf{W}^{(k)}(s+1)(\mathbf{W}^{(k)}(s+1))^{\top}-\mathbf{I}^{(c)}\Vert_{F} -\\
&\scriptstyle  \quad \quad \quad \quad \quad \frac{\Vert \mathbf{W}^{(k)}(s+1)(\mathbf{W}^{(k)}(s+1))^{\top}-\mathbf{I}^{(c)}\Vert_{F}^{2}}{2\Vert\mathbf{W}^{(k)}(s)(\mathbf{W}^{(k)}(s))^{\top}-\mathbf{I}^{(c)}\Vert_{F}}  \Big) \leq \nonumber \\
&\scriptstyle  \sum_{k=1}^{c}\Big(\Vert \mathbf{W}^{(k)}(s)(\mathbf{W}^{(k)}(s))^{\top}-\mathbf{I}^{(c)}\Vert_{F}
-\frac{\Vert \mathbf{W}^{(k)}(s)(\mathbf{W}^{(k)}(s))^{\top}-\mathbf{I}^{(c)}\Vert_{F}^{2}}{2\Vert\mathbf{W}^{(k)}(s)(\mathbf{W}^{(k)}(s))^{\top}-\mathbf{I}^{(c)}\Vert_{F}} \Big) \nonumber
\end{align}
From $\mathbf{P}^{(R)}$, we obtain:
\begin{align}\label{proof10}
&\scriptstyle \sum_{k=1}^{c}\Big( \Vert \mathbf{U}^{(k)}(s+1)\Vert_{F} - \frac{\Vert \mathbf{U}^{(k)}(s+1) \Vert_{F}^{2}}{2\mathbf{U}^{(k)}(s)\Vert_{F}}\Big)\nonumber \\
& \scriptstyle \leq  \sum_{k=1}^{c}\Big( \Vert \mathbf{U}^{(k)}(s)\Vert_{F} - \frac{\Vert \mathbf{U}^{(k)}(s) \Vert_{F}^{2}}{2\mathbf{U}^{(k)}(s)\Vert_{F}}  \Big)
\end{align}

Then, adding Eqs. (\ref{proof6}-\ref{proof10}) to Eq. (\ref{proof5}) on both sides results in the following inequality:
\begin{align}\label{proof11}
&\scriptstyle  \mathbfcal{F}_{1}(s+1)+\mathbfcal{F}_{2}(s+1)+\mathbfcal{F}_{3}(s+1) +\lambda_{2}\sum_{k=1}^{c}\Vert\mathbf{U}^{(k)}(s+1)\Vert_{F} \nonumber \\
&\scriptstyle  +\lambda_{1}\sum_{i=j}^{m}\sum_{k=1}^{c}\Vert\mathbf{V}_{j}^{(k)}(s)\mathbf{W}_{j}^{(k)}(s+1)\Vert_{F} \nonumber \\
&\scriptstyle  +\lambda^{l}\sum_{i=1}^{l}\Vert \mathbf{W}^{i}(s+1)\mathbf{W}^{i\top}(s+1)-\mathbf{I}^{l}\Vert_{F}\nonumber \\
&\scriptstyle  +\lambda_{d}\sum_{j=1}^{d}\Vert \mathbf{W}_{j}(s+1)\mathbf{W}_{j}^{\top}(s+1)-\mathbf{I}_{d}\Vert_{F}\nonumber \\
&\scriptstyle  +\lambda^{(c)}\sum_{k=1}^{c}\Vert \mathbf{W}^{(k)}(s+1)\mathbf{W}^{(k)\top}(s+1)-\mathbf{I}^{(c)}\Vert_{F} \nonumber \\
&\scriptstyle   \leq \mathbfcal{F}_{1}(s)+\mathbfcal{F}_{2}(s)+\mathbfcal{F}_{3}(s)+\lambda_{2}\sum_{k=1}^{c}\Vert\mathbf{U}^{(k)}(s)\Vert_{F} \\
&\scriptstyle  +\lambda_{1}\sum_{i=j}^{m}\sum_{k=1}^{c}\Vert\mathbf{V}_{j}^{(k)}(s)\mathbf{W}_{j}^{(k)}(s)\Vert_{F} \nonumber \\
&\scriptstyle  +\lambda^{l}\sum_{i=1}^{l}\Vert \mathbf{W}^{i}(s)\mathbf{W}^{i\top}(s)-\mathbf{I}^{l}\Vert_{F}  +\lambda_{d}\sum_{j=1}^{d}\Vert \mathbf{W}_{j}(s)\mathbf{W}_{j}^{\top}(s)-\mathbf{I}_{d}\Vert_{F} \nonumber \\
&\scriptstyle  +\lambda^{(c)}\sum_{k=1}^{c}\Vert \mathbf{W}^{(k)}(s)\mathbf{W}^{(k)\top}(s)-\mathbf{I}^{(c)}\Vert_{F} \nonumber
\end{align}
This inequality in Eq. (\ref{proof11}) proves that the objective value is decreased in each iteration.
The Log-Cosh losses in Eq. (\ref{eqn8}) are convex (\cite{gangal2007performance}) 
and the Frobenius norms in Eq. (\ref{eqn8}) are also convex (\cite{negahban2012restricted}). 
Thus, as a sum of convex functions, the objective function in Eq. (\ref{eqn8}) is also convex. 
Therefore, Algorithm \ref{thm1} is guaranteed to converge to the global optimal solution to the formulated constrained regularized optimization problem in Eq. (\ref{eqn7}).
\end{proof}

\noindent \textbf{Time Complexity.} As the formulated optimization problem in Eq. (\ref{eqn7}) is convex, Algorithm \ref{alg1} converges fast (e.g., within tens of iterations only).
In each iteration of our algorithm, computing Steps 3, 4, 5, 6, 8 and 10 is trivial. Steps 7, 9 and 11 can be computed by solving a system of linear equations with quadratic complexity.

\section{Experiments}\label{sec:EXPT}
In this section, we first discuss on the experimental setup and implementation details of our approach.
Then we analyze the experimental results obtained by our method and compare it with the previous state-of-the-art methods in multiple ground navigation scenarios in real-world off-road environments.


\begin{figure}
\centering
\includegraphics[width=0.485\textwidth]{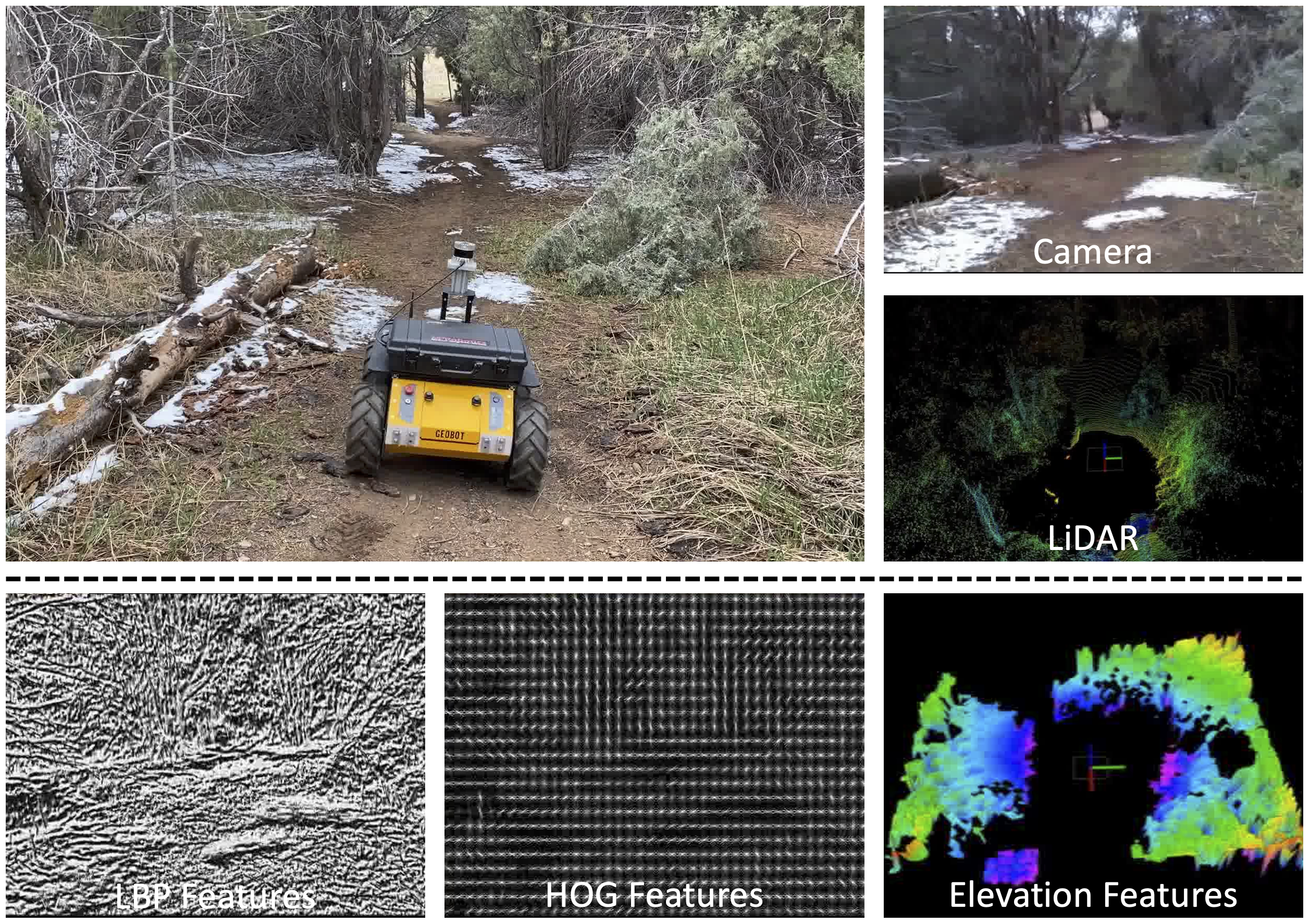}
\caption{
Multi-modal features used in our experimentation to characterize robot's traversing terrain.
}\label{fig:features}
\end{figure}

\subsection{Experimental Setups}
We extensively evaluate our approach using a Clearpath Husky unmanned ground vehicle robot to navigate over a variety of unstructured off-road terrains. 
The robot is equipped with multiple exteroceptive and proprioceptive sensors.
Exteroceptive sensors include an Intel Realsense D435 color camera and an Ouster OS1-64 LIDAR to observe the terrain. 
Proprioceptive sensors including wheel encoders and Microstrain 3DM-GX5-25 Internal Measurement Unit (IMU) are used to measure the robot's internal states while navigating over these terrains. 
All these terrain observations are obtained from the robot and are linearly interpolated to get observations constantly at 15 Hz. 

\begin{figure*}
\centering
\includegraphics[width=0.98\textwidth]{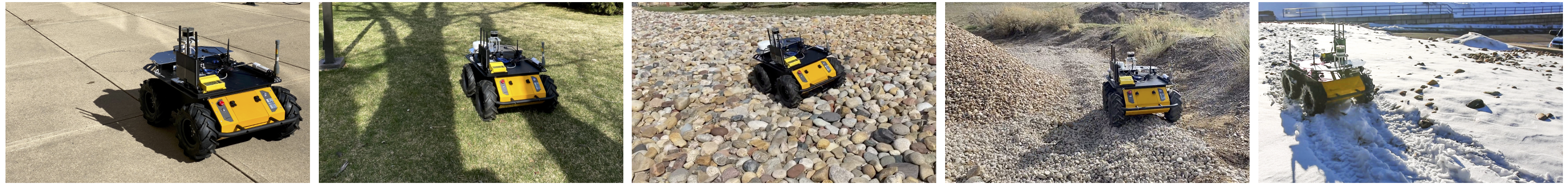}
\caption{Various terrain tracks used for for individual types of terrains. From left: concrete, grass, gravels, large-rocks and snow terrain.
}\label{fig:figure_terrain_1_1}
\end{figure*}

For representing unstructured terrains, we implement multiple visual features extracted from color images to depict the characteristics ot the unstructured terrain. 
These features include Histogram of Oriented Gradients (HOG) (\cite{dalal2005histograms}) to depict the shape of the terrain and Local Binary Patterns (LBP) (\cite{ahonen2006face})  to detail the texture of the terrain. 
To show the various terrain characteristics such as the slope, step height, and terrain normals, we also compute an elevation map (\cite{fankhauser2014robot}) of the terrain using LIDAR data to form a grid-wise robot centric elevation map. 
Apart from these features, we also concatenate sensory readings from the Inertial Measurement Unit (IMU) into a vector of features. 
The various features used in our experiments are visualized in Figure \ref{fig:features}.

In addition, the expected robot behaviors are recorded as demonstrations from an expert human operator 
The actual robot navigational behaviors have been recorded using SLAM techniques (\cite{legoloam2018}).  We use a sequence of five frames (i.e., $c=15$), and the hyperparameters $\lambda_{1}=1$ and $\lambda_{2}=0.1$ for all of our experiments.

We compare our approach against three state-of-the-art approaches. These three approaches can be broadly classified into two categories. The first category is model-based approaches which includes the methods of Model Predictive Path Integral (MPPI)  control (\cite{williams2016aggressive}) and Non-Linear Control Optimization tool (NLOPT) (\cite{nlopt}). The second category of methods includes learning-based approach of terrain representation and apprenticeship learning (TRAL) (\cite{siva2019robot}). 

All the aforementioned methods are quantitatively evaluated against our over various terrains. We define the following evaluation metrics to illustrate the observations:
\begin{itemize}
    \item \emph{Failure Rate}: This metric is defined as the number of times the robot fails to complete the navigation task across a set of experimental trails. A failure is when the robot gets stuck in the terrain or in an obstacle  or flips over in the terrain. Lower values of failure rate indicate better performance and is preferred. 
    \item \emph{Traversal Time}: This metric is defined as the average time taken to complete the navigational task in all the successive runs over the given terrain. Generally, smaller values of traversal time indicate better performance and may be preferred.   
    \item \emph{Inconsistency}: This metrics is defined as the error between the robot's expected navigational  behavior and its actual behavior in terms of robot poses (the linear and angular position). A lower value of inconsistency metric indicate better performance and is preferred.
    \item \emph{Jerkiness}: This metric is defined as the average sum of the acceleration derivatives along all axes, with lower values indicating better performance. Jerkiness indicates how smooth a robot can traverse over a terrain. This is a useful metric for state estimation and SLAM methods that may assume smooth robot motions.
\end{itemize}

\begin{figure*}
\centering
\includegraphics[width=0.98\textwidth]{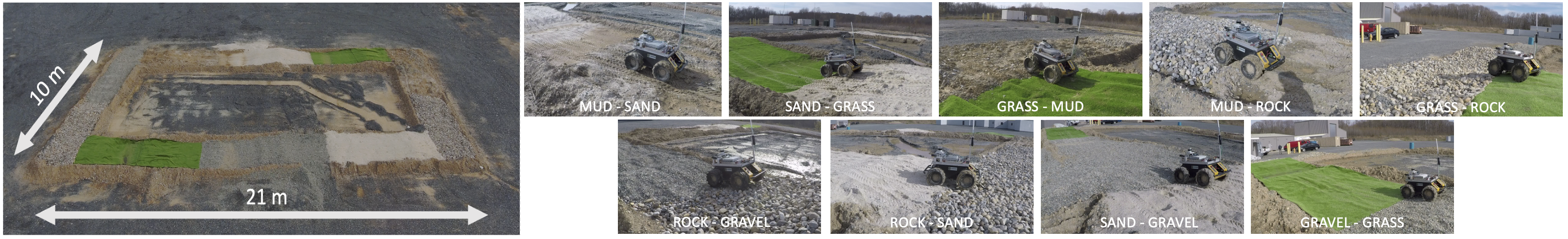}
\caption{
Off-road terrain circuit (left) consisting of multiple terrains of grass, gravels, large rocks, mud and sand. These tracks also show a combination of transitions between these individual terrains (right) and also have a varying slope of $0^{\circ}$-$30^{\circ}$ throughout the track.
}\label{fig:figure_terrain_1_2}
\end{figure*}

\subsection{Terrain Adaptation for Ground Robot Navigation over Unstructured Terrains}
In this set of experiments, we evaluate our proposed approach against its ability to generate terrain aware navigational behaviors. 
The ability to generate terrain aware behaviors are required mainly when 
robots are i) navigating over individual terrain tracks and ii) navigating over complex terrains. 
Therefore, we test our approach along with comparison methods in these two kinds of scenarios. More discussion on each of these experimental scenarios are as follows:

\begin{table*}
\centering
\caption{Quantitative results based on ten runs for scenarios when the robot traverses over \textbf{individual types of unstructured terrain} shown in Fig. \ref{fig:figure_terrain_1_1}.
Successful runs (with no failures) are used to calculate the metrics of traversal time, inconsistency and jerkiness.
Our approach is compared with the methods of MPPI (\cite{williams2016aggressive}), NLOPT (\cite{nlopt}) and TRAL (\cite{siva2019robot}).
}
\label{tab:RTT1_1}
\tabcolsep=0.07cm
\begin{tabular}{ | c| c| c| c| c|| c| c| c| c|| c| c| c| c|| c| c| c| c|}
\hline
\multicolumn{1}{|c|}{} & \multicolumn{4}{c||}{\small{Failure Rate (/10)}} & \multicolumn{4}{c||}{\small{Traversal Time (s)}} & \multicolumn{4}{c||}{\small{Inconsistency}}& \multicolumn{4}{c|}{\small{Jerkiness (m/s$^3$)}} \\
\hline \small{Terrain} & \small{MPPI} & \small{NLOPT} & \small{TRAL} & \small{\textbf{Ours}} & \small{MPPI} & \small{NLOPT} & \small{TRAL} &  \small{\textbf{Ours}} & \small{MPPI} & \small{NLOPT} & \small{TRAL} & \small{\textbf{Ours}} & \small{MPPI} & \small{NLOPT} & \small{TRAL} & \small{\textbf{Ours}}\\ \hline
\small{Concrete} & \small{0} & \small{0} & \small{0} & \small{0} &
\small{12.82} &  \small{17.82} & \small{16.04} & \small{12.49} &
\small{0.65} & \small{0.79} & \small{0.69} & \small{0.56} &
\small{5.13} & \small{4.98} & \small{5.19} & \small{4.35}\\

\small{Grass} & \small{0} & \small{0} & \small{0} & \small{0} &
\small{13.45}& \small{18.02}  & \small{16.39}  & \small{13.32} &
\small{0.61} & \small{0.90} & \small{0.88} & \small{0.62} &
\small{7.82} & \small{8.52} & \small{8.32} & \small{6.98} \\

\small{Gravels} & \small{0} & \small{0} & \small{0} & \small{0} &
\small{13.96}  & \small{16.39} & \small{15.32} & \small{14.84} &
\small{1.21} & \small{1.67} & \small{1.72} & \small{1.03} &
\small{16.11} & \small{16.66} & \small{16.43} & \small{15.09} \\

\small{L.Rock} & \small{0} & \small{1} & \small{0} & \small{0} &
\small{13.45}  & \small{17.87} & \small{16.71} & \small{16.31} &
\small{3.97} & \small{5.64} & \small{3.21} & \small{2.91} &
\small{13.75} & \small{17.61} & \small{14.85} & \small{12.32} \\

\small{Snow} & \small{0} & \small{0} & \small{0} & \small{0} &
\small{14.22}  & \small{18.09} & \small{17.49} & \small{14.08} &
\small{7.18} & \small{9.93} & \small{8.12} & \small{6.64} &
\small{5.12} & \small{5.76} & \small{5.49} & \small{5.42}\\

\hline
\end{tabular}
\end{table*}

\subsubsection{Experimental Analysis on Individual Terrain Tracks}\label{ref:train_start}
In this experimental scenario, we evaluate our proposed approach on various individual types of terrains. That is each of the testing tracks used in this scenario is made of one type of terrain. The length of the tracks are kept approximately ten meters long.
We use five different terrains: concrete, grass, gravels, large-rocks and snow terrains, and are illustrated in Figure \ref{fig:figure_terrain_1_1}. Each of these five terrains has distinctive characteristics and are frequently observed in real-world environments. 
Our approach is trained on driving data collected while the robot is manually being controlled by an expert human operator to traverse over each of the terrains. 
The learned model is then deployed on the robot to generate navigational behaviors as the robot autonomously navigates over the terrain. 
Specifically, our approach is used as a local controller which operates under a local planner. 
Then the evaluation metrics for each of the methods are computed across ten trials on each type of individual terrain tracks.

The quantitative results achieved by our approach and its comparison with other methods are provided in Table \ref{tab:RTT1_1}. 
We observe that nearly all the approaches successfully navigate over all the individual types of terrain with an exception of the NLOPT approach that fails in a single test run on large rocks terrain. 
The presented traversal time is computed by averaging the traversal time across all successful runs, i.e., it excludes the failed trials captured by the failure rate metric.
In terms of traversal time, our approach achieves better traversal time in the terrains of  concrete, grass and snow, followed by the MPPI approach. 
In the terrains of gravels and large-rocks we observe that our approach has a slower traversal time as compared to the MPPI approach and performs the second best. 
Both the NLOPT and TRAL approach both have a slower traversal time over all the terrains with the NLOPT approach performing the slowest throughout all the terrains. 
Also in model-based approaches of MPPI and NLOPT we do not observe the terrain aware nature of the approach, that is there is no adaptation in navigational behaviors as the traversal time remains nearly the same. 
On the other hand, learning-based approaches do show variation in traversal over different terrains. 
Specifically, our approach shows a faster traversal time in plain terrains of concrete, grass and snow, and a slower traversal time on both the gravels and large rocks terrain.

Table \ref{tab:RTT1_1} also presents the inconsistency and the jerkiness metrics. We observe that our approach achieves the best inconsistency throughout all the terrains except the grass terrain, followed by both the MPPI and TRAL approach. 
The MPPI approach has a best inconsistency value in grass terrain. The TRAL approach obtains a better performance to MPPI approach in the large-rocks terrains but otherwise has a poorer performance. 
The NLOPT approach obtains a highest inconsistency metric throughout all the terrains and thus has the least performance in terms of the inconsistency metric.
In general, we observe that learning based methods obtain a better performance over model-based approaches. 
While comparing against the jerkiness metric, we observe that our approach obtains the least value of jerkiness in all terrains but snow. 
The MPPI approach obtains the best performance in snow terrain and trails our approach in all the remaining terrains. 
Both the NLOPT and TRAL approach obtain highest value of jerkiness in specific terrains and perform poorly as compared to both MPPI and our approach.
We observe that both the MPPI approach and our approach perform best in the individual type of terrains across all the metrics with our approach performing the best in most of the scenarios.

\begin{table*}
\centering
\caption{Quantitative results based on ten runs for scenarios when the robot traverses over \textbf{complex terrain tracks} shown in Fig. \ref{fig:figure_terrain_1_2}.
Successful runs (with no failures) are used to calculate the metrics of traversal time, inconsistency and jerkiness.
Our approach is compared with classical control based approaches of MPPI (\cite{williams2016aggressive}), NLOPT (\cite{nlopt}) and TRAL (\cite{siva2019robot}).
}
\label{tab:RTT1_2}
\tabcolsep=0.045cm
\begin{tabular}{ | c| c| c| c| c|| c| c| c| c|| c| c| c| c|| c| c| c| c|}
\hline
\multicolumn{1}{|c|}{} & \multicolumn{4}{c||}{\small{Failure Rate (/10)}} & \multicolumn{4}{c||}{\small{Traversal Time (s)}} & \multicolumn{4}{c||}{\small{Inconsistency}}& \multicolumn{4}{c|}{\small{Jerkiness (m/s$^3$)}} \\
\hline \small{Terrain} & \small{MPPI} & \small{NLOPT} & \small{TRAL} & \small{\textbf{Ours}} & \small{MPPI} & \small{NLOPT} & \small{TRAL} &  \small{\textbf{Ours}} & \small{MPPI} & \small{NLOPT} & \small{TRAL} & \small{\textbf{Ours}} & \small{MPPI} & \small{NLOPT} & \small{TRAL} & \small{\textbf{Ours}}\\ \hline
\small{Mud-Sand} & \small{0} & \small{0} & \small{0} & \small{0} &
\small{26.61} & \small{31.11} & \small{41.08} & \small{36.17} & 
\small{- } & \small{ - } & \small{ - } & \small{ - } & 
\small{ - } & \small{ - } & \small{ - } & \small{ -} \\

\small{Sand-Grass} & \small{ 0 } & \small{ 0 } & \small{ 0 } & \small{ 0 } & 
\small{ 23.69 } & \small{ 31.78 } & \small{ 36.57 } & \small{ 28.79 } & 
\small{ - } & \small{ - } & \small{ - } & \small{ - } & 
\small{ - } & \small{ - } & \small{ - } & \small{ - }\\

\small{Grass-Mud}  & \small{ 0 } & \small{ 0 } & \small{ 0 } & \small{ 0 } & \small{
18.83 } & \small{ 20.69 } & \small{ 37.35 } & \small{ 26.74 } & \small{
- } & \small{ - } & \small{ - } & \small{ - } & \small{
- } & \small{ - } & \small{ - } & \small{ - }\\

\small{Mud-Rock} & \small{ 0 } & \small{ 0 } & \small{ 0 } & \small{ 0 } & \small{
19.31 } & \small{ 21.86 } & \small{ 39.88 } & \small{ 33.32 } & \small{
- } & \small{ - } & \small{ - } & \small{ - } & \small{
- } & \small{ - } & \small{ - } & \small{ - }\\

\small{Rock-Sand} & \small{ 0 } & \small{ 0 } & \small{ 0 } & \small{ 0 } & 
\small{ 20.44 } & \small{ 22.67 } & \small{ 36.08 } & \small{ 36.59 } & 
\small{ - } & \small{ - } & \small{ - } & \small{ - } & 
\small{ - } & \small{ - } & \small{ - } & \small{ - }\\

\small{Sand-Gravel} & \small{ 0 } & \small{ 0 } & \small{ 0 } & \small{ 0 } & 
\small{ 20.93 } & \small{ 23.98 } & \small{ 41.9 } & \small{ 43.02 } & 
\small{ - } & \small{ - } & \small{ - } & \small{ - } & 
\small{ - } & \small{ - } & \small{ - } & \small{ - }\\

\small{Gravel-Grass} & \small{ 0 } & \small{ 0 } & \small{ 0 } & \small{ 0 } & 
\small{ 22.38 } & \small{ 26.04 } & \small{ 39.17 } & \small{ 29.07 } & 
\small{ - } & \small{ - } & \small{ - } & \small{ - } & 
\small{ - } & \small{ - } & \small{ - } & \small{ - }\\

\small{Grass-Rock} & \small{ 0 } & \small{ 0 } & \small{ 0 } & \small{ 0 } & 
\small{ 19.81 } & \small{ 26.55 } & \small{ 36.15 } & \small{ 35.82 } & 
\small{ - } & \small{ - } & \small{ - } & \small{ - } & 
\small{ - } & \small{ - } & \small{ - } & \small{ - }\\

\small{Rock-Gravel} & \small{ 0 } & \small{ 0 } & \small{ 0 } & \small{ 0 } & 
\small{17.62 } & \small{ 27.03 } & \small{ 34.92 } & \small{ 42.41 } & 
\small{ - } & \small{ - } & \small{ - } & \small{ - } & 
\small{ - } & \small{ - } & \small{ - } & \small{ - } \\

\hline
\end{tabular}
\end{table*}

\begin{figure*}
\centering
\includegraphics[width=0.98\textwidth]{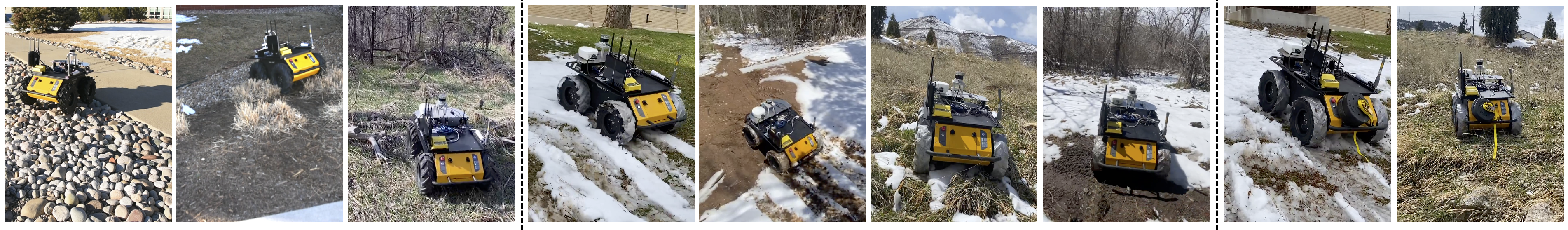}
\caption{Various terrain tracks used to evaluate the self-reflection capability of the robot. From left: mixed terrain I (MT-I), mixed terrain II (MT-II), forest, reduced-traction snow-grass (RT-Snow-Grass), reduced-traction mud-snow (RT-Snow-Mud), reduced-traction hill (RT-Hill), reduced-traction forest (RT-Snow-Forest), reduced-traction payload snow-grass (RT-P-Snow-Grass) and reduced-traction payload hill terrain  (RT-P-Hill).
}
\label{fig:figure_terrain_2}
\end{figure*}

\begin{table*}
\centering
\caption{Quantitative results based on ten runs for scenarios shown in Fig. \ref{fig:figure_terrain_2}
used to evaluate the self-reflection capability of the robot.
Successful runs (with no failures) are used to calculate the metrics of traversal time, inconsistency and jerkiness.
Our approach is compared with classical control based approaches of MPPI (\cite{williams2016aggressive}), NLOPT (\cite{nlopt}) and TRAL (\cite{siva2019robot}).}
\label{tab:RTT2_1}
\tabcolsep=0.01cm
\begin{tabular}{ | c| c| c| c| c|| c| c| c| c|| c| c| c| c|| c| c| c| c|}
\hline
\multicolumn{1}{|c|}{} & \multicolumn{4}{c||}{\small{Failure Rate (/10)}} & \multicolumn{4}{c||}{\small{Traversal Time (s)}} & \multicolumn{4}{c||}{\small{Inconsistency}}& \multicolumn{4}{c|}{\small{Jerkiness (m/s$^3$)}} \\ 
\hline \small{Terrain} & \small{MPPI} & \small{NLOPT} & \small{TRAL} & \small{\textbf{Ours}} & \small{MPPI} & \small{NLOPT} & \small{TRAL} &  \small{\textbf{Ours}} & \small{MPPI} & \small{NLOPT} & \small{TRAL} & \small{\textbf{Ours}} & \small{MPPI} & \small{NLOPT} & \small{TRAL} & \small{\textbf{Ours}}\\ \hline

\small{MT-I} & \small{ 0 } & \small{ 0 } & \small{ 0 } & \small{ 0 } & 
\small{ 14.51 } & \small{  17.68 } & \small{ 16.84 } & \small{ 13.74 } &
\small{ 0.65 } & \small{ 0.79 } & \small{ 0.69 } & \small{ 0.56 } & 
\small{ 5.68 } & \small{ 5.57 } & \small{ 5.13 } & \small{ 4.95 }\\

\small{MT-II} & \small{ 1 } & \small{ 2 } & \small{ 1 } & \small{ 0 } & 
\small{ 13.44 } & \small{ 17.48 } & \small{ 16.74 } & \small{ 13.64 } & 
\small{ 0.78 } & \small{ 0.98 } & \small{ 0.88 } & \small{ 0.71 } & 
\small{ 7.66 } & \small{ 8.75 } & \small{ 8.35 } & \small{ 6.52 }\\

\small{Forest} & \small{ 2 } & \small{ 5 } & \small{ 3 } & \small{ 1 } & 
\small{ 14.67 } & \small{ 13.67 } & \small{ 15.29 } & \small{ 15.96 } & 
\small{ 1.21 } & \small{ 1.67 } & \small{ 1.72 } & \small{ 1.11 } & 
\small{ 16.25 } & \small{ 16.62 } & \small{ 16.55 } & \small{ 16.78}\\
\hline
\small{RT-Snow-Grass} & \small{ 3 } & \small{ 4 } & \small{ 4 } & \small{ 2 } & 
\small{ 13.24 } & \small{ 17.66 } & \small{ 16.18 } & \small{ 13.72 } & 
\small{ 4.61 } & \small{ 5.90 } & \small{ 5.88 } & \small{ 4.11 } & 
\small{ 9.12 } & \small{ 9.52 } & \small{ 9.32 } & \small{ 8.98 }\\

\small{RT-Snow-Mud} & \small{ 5 } & \small{ 7 } & \small{ 4 } & \small{ 2 } & 
\small{ 13.28 } & \small{ 17.84 } & \small{ 16.10 } & \small{ 12.03 } & 
\small{ 3.97 } & \small{ 5.64 } & \small{ 3.21 } & \small{ 3.01 } & 
\small{ 13.14 } & \small{ 17.15 } & \small{ 14.12 } & \small{ 12.35 }\\

\small{RT-Hill} & \small{ 3 } & \small{ 6 } & \small{ 3 } & \small{ 2 } & 
\small{ 13.07 } & \small{ 15.85 } & \small{ 16.42 } & \small{ 13.91 } & 
\small{ 14.98 } & \small{ 15.23 } & \small{ 14.91 } & \small{ 13.07 } & 
\small{ 13.75 } & \small{ 17.61 } & \small{ 14.85 } & \small{ 12.32}\\

\small{RT-Snow-Forest}  & \small{ 0 } & \small{ 0 } & \small{ 0 } & \small{ 0 } & 
\small{ 13.55 } & \small{ 16.63 } & \small{ 14.39 } & \small{ 14.34 } & 
\small{ 7.18 } & \small{ 9.93 } & \small{ 8.12 } & \small{ 6.62 } & 
\small{ 6.72 } & \small{ 6.88 } & \small{ 6.92 } & \small{ 6.47 } \\

\hline

\small{RT-P-Snow-Grass} & \small{ 4 } & \small{ 7 } & \small{ 5 } & \small{ 2 } & 
\small{ 13.96 } & \small{ 16.39 } & \small{ 15.32 } & \small{ 13.84 } & 
\small{ 9.12 } & \small{ 9.96 } & \small{ 9.57 } & \small{ 8.54 } & 
\small{ 16.11 } & \small{ 16.66 } & \small{ 16.43 } & \small{ 16.09}\\

\small{RT-P-Hill} & \small{ 5 } & \small{ 8 } & \small{ 6 } & \small{ 4 } & 
\small{ 14.59 } & \small{ 15.87 } & \small{ 14.38 } & \small{ 15.77 } & 
\small{ 11.18 } & \small{ 15.93 } & \small{ 14.12 } & \small{ 10.62 } & 
\small{13.12 } & \small{ 17.56 } & \small{ 14.86 } & \small{ 12.51}\\

\hline
\end{tabular}
\end{table*}

\subsubsection{Experimental Analysis on Complex Terrain Tracks}
In the second experimental scenario, we evaluate our approach when robots navigate over complex off-road unstructured environments. To replicate the challenges of these real-world complex terrains we built an off-road terrain circuit track. The testing track has different combination of terrains of varying slopes of $0^{\circ}$-$30^{\circ}$ throughout the track.
We then test our robot in this testing track by taking two terrains at once to evaluate how different approaches perform when there is a transition between terrains.
We then have multiple complex terrains made using a combination of grass, gravels, large rocks, mud and sand terrains. 
An aerial view of this testing track along with the multiple complex terrains is shown in \ref{fig:figure_terrain_1_2}. 
The dimension of this testing track is 21 m x 10 m and each segment of terrain making up the track is of length 5 meters long. 
Therefore, each segment of complex terrain track (i.e., two terrain tracks combined) is either of 10 meters or 8 meters depending on the section of track the robot is navigating over. 
The same learned model used in the previous experimental scenario is deployed in this experiment with no additional training. 
In this experiments we compute the evaluation metrics based on five trial runs on each terrain track.

The quantitative results obtained by our approach and its comparison with other state of the art methods are provided in 
Table \ref{tab:RTT1_2}.

\subsection{Self-Reflection for Ground Navigation}
In this set of experiments, we aim to evaluate all the approaches in challenging scenarios where the robot would need the essential ability of self-reflection. 
Specifically, we evaluate the performance in scenarios when robots experience a change in their own functionality. 
Therefore, we introduce in robots three categories of setbacks.W
Each of these setbacks are used to increase the difficulty of robot navigation in a terrain and would increasingly demand the ability of self-reflection. These three setbacks are each tested in specific terrains each listed below:
\begin{itemize}

\item We \emph{over-inflate} all the robot tyres and with this setback robot generally has a higher center of mass and reduced traction with the terrain. 
This setback also causes increased jerkiness in the robot due to less damping from the over-inflated tyres. 
With this setback we test the robots in three different terrains of i) mixed terrain I (MT-I), ii) mixed terrain II (MT-II) and iii) the forest terrain.
None of the three terrains have any significant slope. 

\item We further reduce robot tyre traction by completely \emph{deflating and duct taping the tyres}. 
With this setback the wheel traction reduces significantly and the robot has a lower center of mass. 
This setback in robot is evaluated in four different terrains of i) reduced traction snow-grass (RT-Snow-Grass), ii) reduced traction snow-mud (RT-Snow-Mud), iii) reduced traction hill (RT-Hill) and iv) reduced traction snow-forest (RT-Snow-Forest).
Each of these terrains but the snow-forest terrain have a varying slope of $0^\circ$- $45^\circ$ which further enhances the affect of this setback.  

\item Finally, we introduce in the robot the setback of \emph{reduced tyre pressure, reduced wheel traction and increased payload}. That is on top of the completely deflated and duct taped tyres, we further add a payload  of fifteen pounds.  
This setback is evaluated in two different terrain of i) reduced traction payload snow-grass (RT-P-Snow-Grass) and ii) reduced traction payload hill (RT-P-Hill). Both of these terrains have a varying slope of $0^\circ$- $45^\circ$.
\end{itemize}

These nine different testing terrains are illustrated in Figure \ref{fig:figure_terrain_2}. 
No additional training is performed for this set of experiments and we use the same model from Section \ref{ref:train_start}. Besides, all the other parameters for deployment are kept the same. In this set of experiments the evaluation metrics for each of the methods are calculated based on ten trials of robot runs on each of the terrains. 

The quantitative results obtained by our approach and its comparison with other approaches are presented in Table \ref{tab:RTT1_2}. 
For the first category of setback, we observe that the failure rates is less in the tracks of mixed terrains and the highest failure is observed in the forest terrain.  
Our approach achieves the least failure rate in all of these three different terrains followed by the MPPI and TRAL approach. The NLOPT approach achieves a higher failure rate in both MT-II and forest terrain and is the poor performer amongst the different approaches. 
In terms of the traversal time, we do observe that our approach obtains a lesser traversal time in the terrains of MT-I but also has a slowest traversal time in the forest terrain. This again shows the ability of our approach to adapt its navigational behaviors as the terrain becomes more and more cluttered. The MPPI approach performs better in the MT-II terrain and the surprisingly we see that the NLOPT approach has the fastest traversal time in the Forest terrain which otherwise performs the slowest in the remaining two terrains.  In terms of the inconsistency metric our approach with the ability to generate consistent navigational behaviors is able to obtain the least inconsistency metric overall followed by the  MPPI approach. Both TRAL and NLOPT approach have high inconsistency values in respective terrains with TRAL approach achieving a poor performance in the Forest terrain.  In terms of the jerkiness metric, we observe that all the approaches have an increased value of jerkiness indicating the cluttered nature of terrain in a forest environment. Our approach has a lower jerkiness measure in MT-I and MT-II terrain but also performs poorly in the Forest terrain.  The MPPI approach obtains the second best performance in both the MT-I and MT-II terrain but also obtains the best performance in the Forest terrain. Both the NLOPT and TRAL approach perform poorly and have a high jerkiness metric overall. 

In the second category of setbacks, we observe an increased failure rate as compared to the previous terrains. Especially in the terrains with slope, i.e., RT-Snow-Mud, RT-Hill  and RT-Snow-Grass, failure rate is high as compared to RT-Snow-Forest which has no slope.  Our approach has lesser failure rate throughout all the terrains in this category and the NLOPT approach performs the worst.  In terms of traversal time, we do see that both MPPI performs better than our approach in all but RT-Snow-Mud terrain. Our approach trails MPPI approach by a small window but does show terrain aware behaviors as the velocities differ with the terrain as opposed to the MPPI approach. The approach of TRAL has a poorer performance as compared to MPPI and our approach, and the NLOPT obtains the highest traversal time values indicating the slowest traversal in all the terrains. For the inconsistency metric, we observe that our approach obtains the least value of inconsistency metric followed by both MPPI and TRAL approach with MPPI performing better in RT-Snow-Grass and RT-Snow-Forest and TRAL performing well in the remaining terrains. The NLOPT approach obtains the highest inconsistency value indicating a poor performance compared to other approaches. When observing the jerkiness metrics obtained by the different approaches, we see a similar trend of our approach performing the best in all the different terrains follows by both MPPI and TRAL approach. The NLOPT approach has the lowest jerkiness value in all the terrains but RT-Snow-Forest signaling a poor performance overall. Additionally, we see a lesser value of jerkiness over RT-Snow-Forest terrain as compared to the Forest terrain. This indicates that when covered with snow, much of the cluttered nature of the forest terrain is subdued.

Finally, we evaluate and compare our approach in the third category of setbacks. Having reduced tyre traction and payload especially on terrains with slope is very challenging for mobile robots \cite{grand2004stability}. With these challenges we observe that the failure rates are very high. Specifically in the RT-P-Hill terrain we see that all the approaches fail in nearly half the runs. This shows that with increased setbacks, it is really hard for robots to achieve their expected behaviours. We do observe our approach to have a slightly better performance over the other approaches followed by the MPPI, TRAL and NLOPT approaches, in that order. In terms of traversal time, we see mixed results. We observe that our approach obtains least traversal time in RT-P-Snow-Grass terrain where as is slow in RT-P-Hill terrain. In the RT-P-Hill terrain the TRAL approach obtains the lowest traversal time. The MPPI approach obtains the second best performance throughout and the NLOPT approach having the slowest traversal time in both the terrains.
In terms of inconsistency we observe that our approach has a lower inconsistency on both the terrains followed by the MPPI approach and then the TRAL approach. Again, we see that the inconsistency values obtained by the NLOPT approach is higher than the other approaches. A similar trend is seen in  terms of the jerkiness metric as well. We see our approach has a better performance followed by the MPPI approach in all the terrains. Both the TRAL and NLOPT approaches have a higher value of jerkiness suggesting a poor performance. 

Across all the three different category of setbacks, it is observed that by generating inconsistent navigational behaviors and lowering inconsistency values while navigating over various terrains, our approach is able to reduce the number of failures by trading-off the traversal time in some scenarios. 

\begin{figure*}
  \subfigure[{\footnotesize Performance on Various Frame Sequences}]{
    \label{fig:sequence} 
    \begin{minipage}[b]{0.2\textwidth}
      \centering
        \includegraphics[width=1\textwidth]{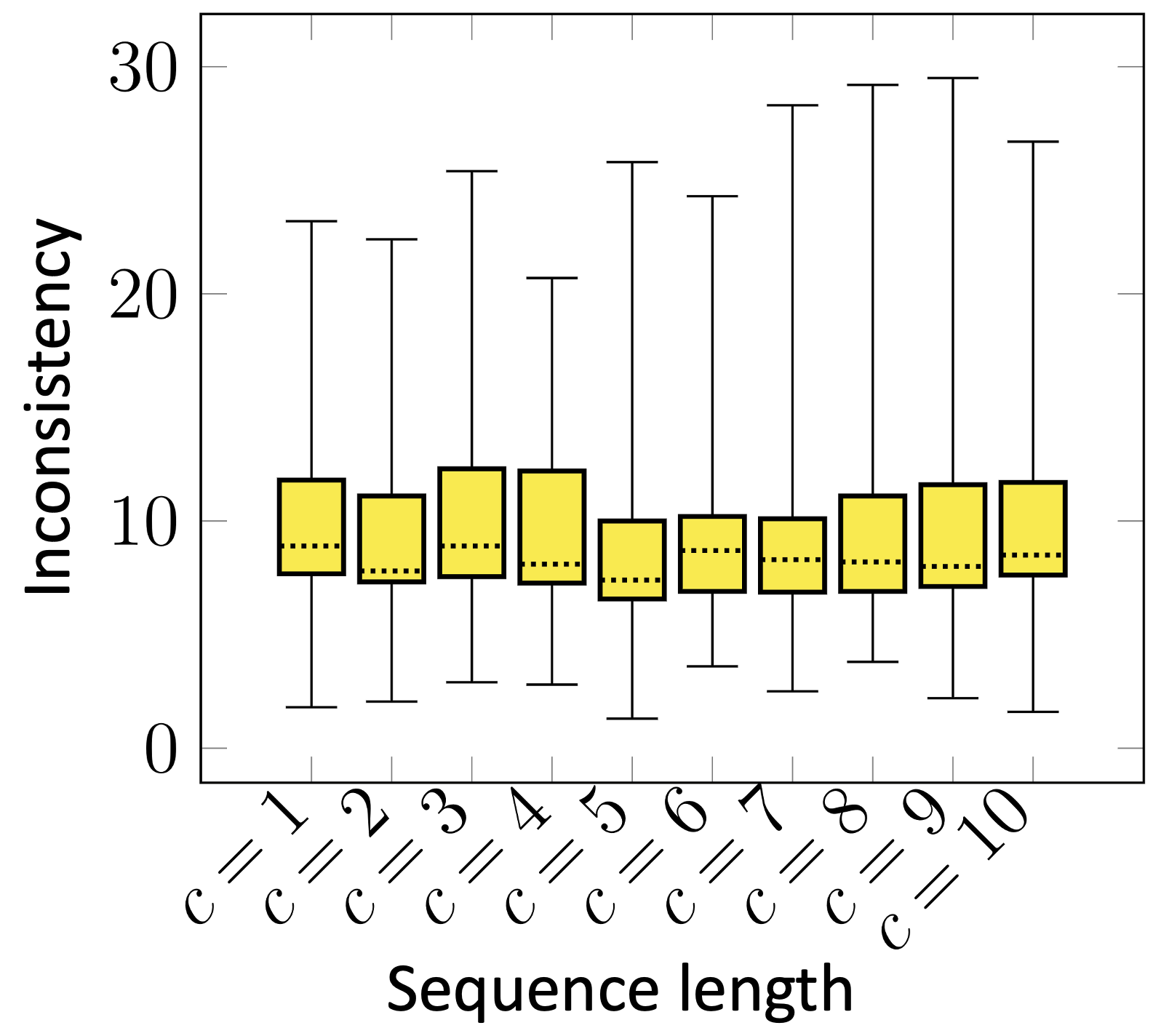}
        \centering
    \end{minipage}}
   \subfigure[{\scriptsize Hyper-parameter Analysis}]{
    \label{fig:hyper}
    \begin{minipage}[b]{0.22\textwidth}
      \centering
        \includegraphics[width=0.9\textwidth]{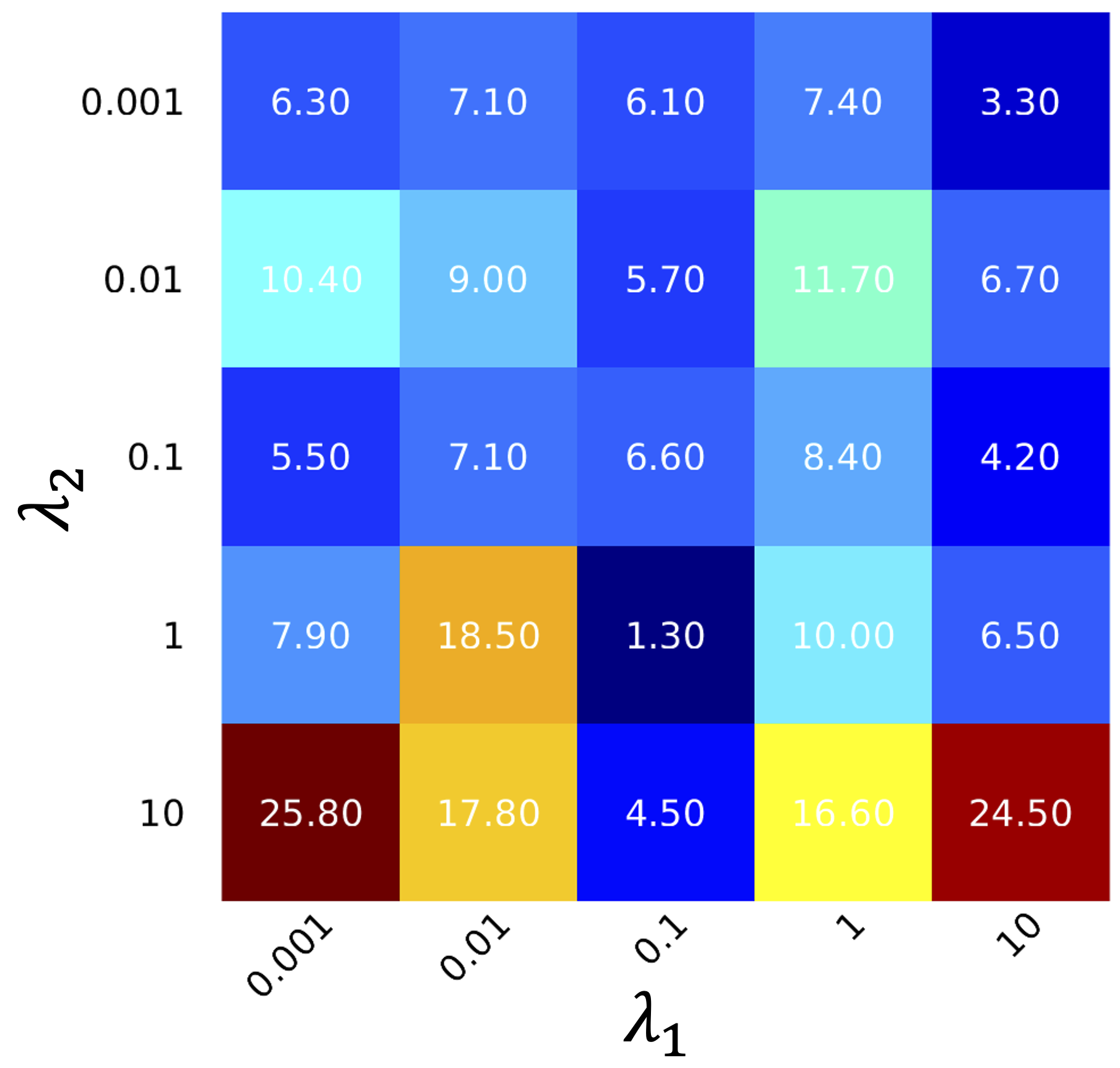}
        \centering
    \end{minipage}}
       \subfigure[{\scriptsize Historical Time Steps}]{
    \label{fig:prev}
    \begin{minipage}[b]{0.25\textwidth}
      \centering
        \includegraphics[width=0.85\textwidth]{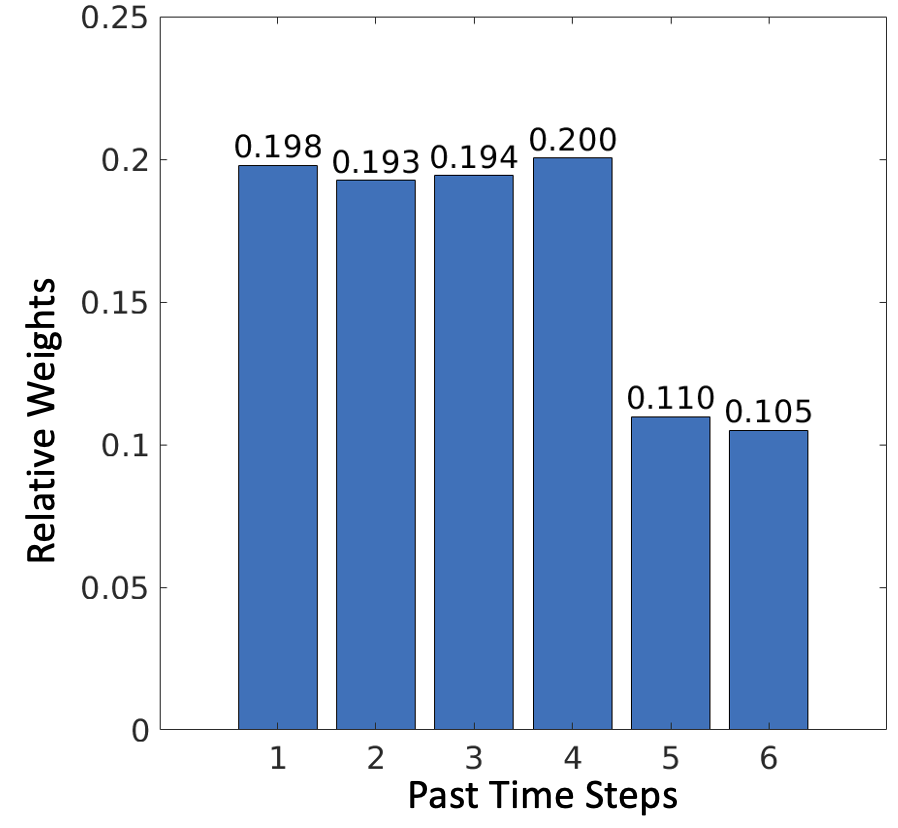}
        \centering
    \end{minipage}}
       \subfigure[{\scriptsize Discriminative Feature Modalities}]{
    \label{fig:feature_analysis}
    \begin{minipage}[b]{0.2\textwidth}
      \centering
        \includegraphics[width=1.35\textwidth]{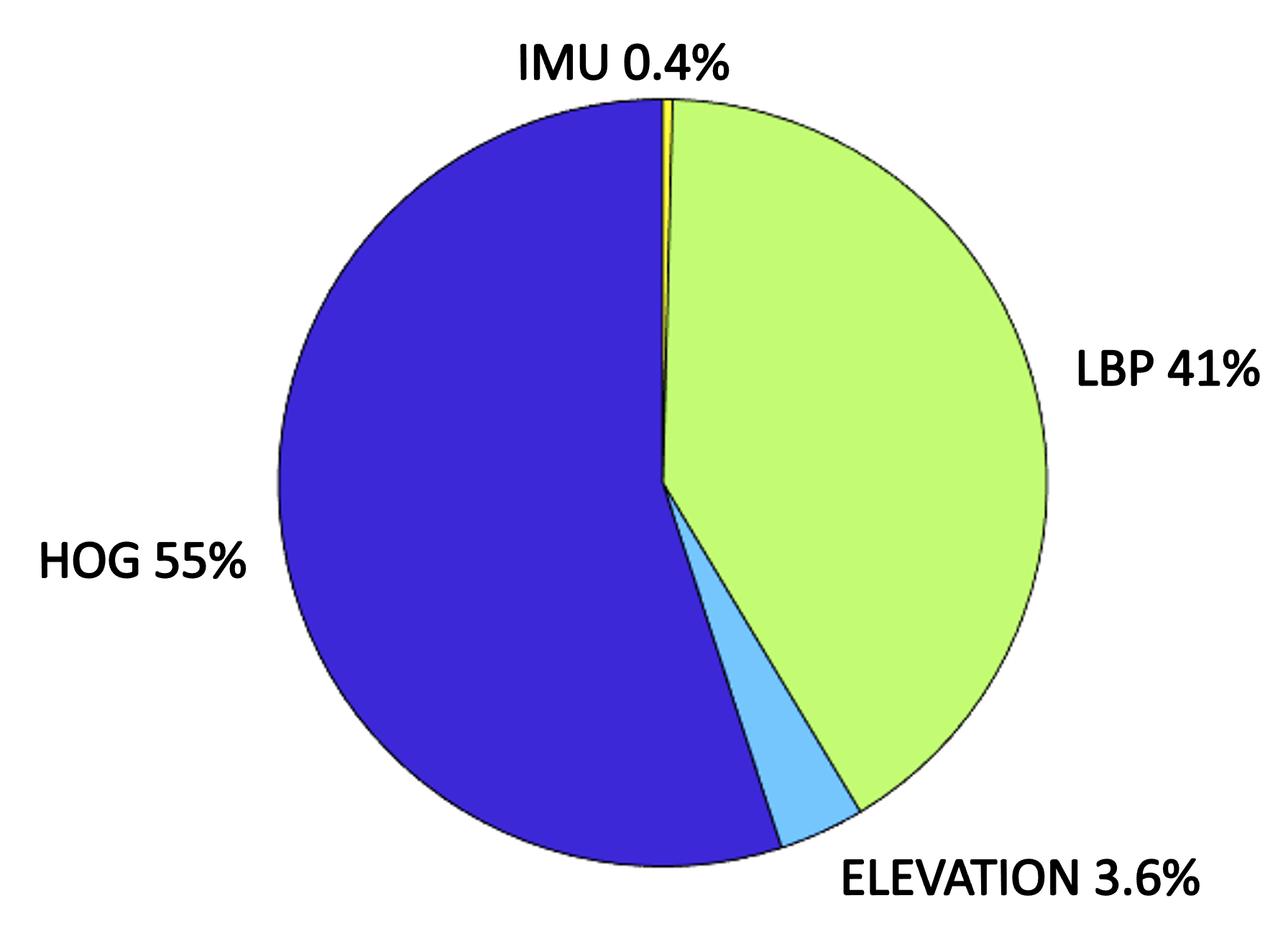}
        \centering
    \end{minipage}
    }
  \caption{Analysis of various parameters used in our experimental runs.
  }  \label{fig:parameter_analysis} 
\end{figure*}

\subsection{Discussion}
We further discuss on the additional experimental results to investigate the characteristics of our approach. These investigations derive the basis for the various parameters used in our experimental runs. This includes performance on various frame sequences, hyperparameter analysis, relative weights of each time step and the discriminative feature modalities.

\subsubsection{Performance on Various Frame Sequences:}
Our approach uses a sequence of historical frames '$c$', to generate terrain aware behavior and self-reflective consistent behaviors. 
Figure \ref{fig:sequence} uses a boxplot to show the distribution of inconsistency value for each value of historical frame sequence $c=1,\dots,10$; under various values of hyperparameters $\lambda_{1}\in [0.001,10]$ and $\lambda_{2}\in [0.001,10]$. The boxplot displays a distribution of results based on a five number summary (i.e., minimum, first quartlie, medium, thirds quartile, and maximum). 
It is observed that our approach performs well in specific values of $c=\{1,5,10\}$, and we observe the inconsistency metric being worst at $c=\{6,8\}$. The best inconsistency value of 1.3 is observed when $c=5$, that is the barplot with lowest minimum. 
This value of historical frame sequences in used in all of our experiments.

\subsubsection{Hyper-parameter Analysis:}
Our approach utilizes two hyper-parameters $\lambda_{1}$ and $\lambda_{2}$ that balances the amount of loss from our objective function and the regularization terms. Figure \ref{fig:hyper} depicts how the inconsistency metric changes with each of these hyperparameter values for five historical frame sequences, i.e., $c=5$. 
In general it is seen that the inconsistency metric changes significantly with variations in $\lambda_{2}$ as compared to $\lambda_{1}$, signifying the importance of modeling historical time steps. 
Altogether, it is observed that $\lambda_{1}\in (0.01,1)$ an $\lambda_{2} \in (0.01,10)$ result in a good performance. The best results are observed with $\lambda_{1}=0.1$ and $\lambda_{2}=1$. These are the fixed hyperparameters were used during execution for all the experiments.

\subsubsection{Historical Time Steps:} 
Our approach models features and behavior difference from a sequences of historical time steps to perform self-reflection. 
We therefor analyze the importance of each time steps from the sequence of frames in generating robot's navigational behaviors. 
Figure \ref{fig:prev} presents the importance of each time steps in the historical sequence of time steps. 
It is observed that the terrain features and behavior differences from most recent time steps have an higher importance for generating navigational behaviors. We observed that the present and past three time steps are equally important with the third past time step  having a weight of $20\%$.
The fourth and fifth past time step have a smaller importance of nearly $10\%$.

\subsubsection{Discriminative Feature Modalities:} 
Our approach has the ability to automatically estimate the importance of the various feature modalities. The results from evaluating the modality importance is demonstrated in Figure \ref{fig:feature_analysis}. It is observed that the two feature modalities, i.e., HOG and LBP features are most discriminative in performing self-reflective terrain adaptation. These two features combined together account for relative importance of $96\%$.  The IMU features and elevation features are comparatively not much important.

\section{Conclusion}\label{sec:CONC}
In this paper, we introduce a novel
method of self-reflective terrain-aware adaptation for autonomous ground robots 
to generate consistent behaviors to navigate over unstructured off-road environments.
Formulated in a unified regularized constrained optimization framework,
our method is also able to identify discriminative terrain features 
and fuse them to perform effective terrain adaptation.
Additionally, we design a new algorithm that addresses the formulated optimization  problem, which holds a theoretical guarantee to converge to the global optimal solution.
To evaluate our approach, we conduct extensive experiments 
using physical ground robots with varying functionality changes over diverse unstructured off-road terrains.
Experimental results have shown that our self-reflective terrain-aware adaptation method
 outperforms previous and baseline techniques and enables ground robots to generate consistent behaviors
when navigating in off-road environments.


\bibliographystyle{SageH}
\bibliography{references}

\end{document}